\documentclass{article}
\usepackage{fullpage}
\usepackage[utf8]{inputenc}
\usepackage{graphicx}
\usepackage{xcolor}
\usepackage[labelformat=simple]{subcaption}

\usepackage[utf8]{inputenc} % allow utf-8 input
\usepackage[T1]{fontenc}    % use 8-bit T1 fonts
\usepackage{hyperref}       % hyperlinks
\usepackage{url}            % simple URL typesetting
\usepackage{booktabs}       % professional-quality tables
\usepackage{amsfonts,amsmath,amsthm}       % blackboard math symbols

\usepackage{nicefrac}       % compact symbols for 1/2, etc.
\usepackage{microtype}      % microtypography

\usepackage[ruled,vlined]{algorithm2e}

\usepackage{tikz-qtree,tikz-qtree-compat}
\usepackage{tikz}

\usepackage{wrapfig}

\definecolor{Cluster0color}{RGB}{237,225,229}
\definecolor{Cluster1color}{RGB}{118,185,223}
\definecolor{Cluster2color}{RGB}{17,86,70}

\DeclareMathOperator*{\argmin}{arg\,min}
\DeclareMathOperator*{\argmax}{arg\,max}

\newcommand{\smallpar}[1]{\medskip\noindent\textbf{#1.}\hspace{.5ex}}

\newtheorem{theorem}{Theorem}
\newtheorem{definition}{Definition}
\newtheorem{corollary}{Corollary}
\newtheorem{claim}{Claim}
\newtheorem{proposition}{Proposition}

\newcommand{\R}{\mathbb{R}}
\newcommand{\E}{\mathbb{E}}

\newcommand{\norm}[1]{\left\lVert#1\right\rVert}
\newcommand\inner[2]{\langle #1, #2 \rangle}
\newcommand{\vectx}{\mathbf{x}}

\newcommand{\vectv}{\mathbf{v}}
\newcommand{\vectmu}{\boldsymbol{\mu}}

\newcommand{\cost}{\mathrm{cost}}
\newcommand{\surcost}{\widetilde{\cost}}
\newcommand{\dist}{\mathsf{dist}}
\newcommand{\cX}{\mathcal{X}}
\newcommand{\cD}{\mathcal{D}}
\newcommand{\centersSet}{\mathcal{M}}
\newcommand{\sklearn}{\texttt{sklearn}}
\newcommand{\cart}{\texttt{CART}}
\newcommand{\imm}{\texttt{IMM}}
\newcommand{\exkmc}{\texttt{ExKMC}}
\newcommand{\cltree}{\texttt{CLTree}}
\newcommand{\cubt}{\texttt{CUBT}}
\newcommand{\vectp}{\mathbf{p}}

\usepackage{colortbl, xcolor}
\usepackage{booktabs, cellspace, multirow, tabularx}
    \newcolumntype{R}{>{\raggedright\arraybackslash}X}
    \addparagraphcolumntypes{R}
\usepackage{ragged2e}

\SetAlFnt{\footnotesize}

\SetAlCapSty{xAlCapSty}

\SetCommentSty{xCommentSty}

\SetNlSty{mynlfont}{}{} 

\LinesNumbered

\SetSideCommentRight

\DontPrintSemicolon

\RestyleAlgo{algoruled}

\usepackage{hyperref}%
\definecolor{bluish-green}{HTML}{006E53}
\hypersetup{%
   breaklinks,%
   colorlinks=true,%
   urlcolor=[rgb]{0.25,0.0,0.0},%
   linkcolor=[rgb]{0.4,0.0,0.2},%
   citecolor=bluish-green,%
   filecolor=[rgb]{0,0,0.4}, anchorcolor=[rgb]={0.0,0.1,0.2}%
}

\newcommand{\xvar}[1]{\textsf{#1}}

\newcommand{\xfunc}[1]{\texttt{#1}}
\title{\texttt{ExKMC}: Expanding Explainable $k$-Means Clustering}

\author{
  Nave Frost\quad\quad\\
  Tel Aviv University\quad\quad\\
  \texttt{\quad navefrost@mail.tau.ac.il\quad}
  \and
    Michal Moshkovitz\\
    University of California, San Diego\\
    \texttt{mmoshkovitz@eng.ucsd.edu}
  \and
    Cyrus Rashtchian\\
    University of California, San Diego\\
  \texttt{crashtchian@eng.ucsd.edu}
}
\date{}

\begin{document}

\maketitle

\begin{abstract}
Despite the popularity of explainable AI, there is limited work on effective methods for unsupervised learning. We study algorithms for $k$-means clustering, focusing on a trade-off between explainability and accuracy. Following prior work, we use a small decision tree to partition a dataset into $k$ clusters. This enables us to explain each cluster assignment by a short sequence of single-feature thresholds. While larger trees produce more accurate clusterings, they also require more complex explanations. To allow flexibility, we develop a new explainable $k$-means clustering algorithm, \exkmc{}, that takes an additional parameter $k' \geq k$ and outputs a decision tree with $k'$ leaves. We use a new surrogate cost to efficiently expand the tree and to label the leaves with one of $k$ clusters. We prove that as $k'$ increases, the surrogate cost is non-increasing, and hence, we trade explainability for accuracy. 
Empirically, we validate that \exkmc{} produces a low cost clustering, outperforming both standard decision tree methods and other algorithms for explainable clustering. Implementation of \exkmc{} available at \url{https://github.com/navefr/ExKMC}.
\end{abstract}

\section{Introduction}

The bulk of research on explainable machine learning studies how to interpret the decisions of supervised learning methods, largely focusing on feature importance in black-box models~\cite{arrieta2020explainable, deutch2019constraints, lipton2018mythos,lundberg2017unified, molnar2019, murdoch2019interpretable, ribeiro2016should,  rudin2019stop}. To complement these efforts, we study explainable algorithms for clustering, a canonical example of unsupervised learning.  Most clustering algorithms operate iteratively, using global properties of the data to converge to a low-cost solution.
For center-based clustering, the best explanation for a cluster assignment may simply be that a data point is closer to some center than any others. While this type of explanation provides some insight into the resulting clusters, it obscures the impact of individual features, and the cluster assignments often depend on the dataset in a complicated way.

Recent work on explainable clustering goes one step further by enforcing that the clustering be derived from a binary threshold tree~\cite{bertsimas2018interpretable,chen2016interpretable,dasgupta2020explainable,fraiman2013interpretable,ghattas2017clustering,liu2005clustering}. Each node is associated with a feature-threshold pair that recursively splits the dataset, and labels on the leaves correspond to clusters. Any cluster assignment can be explained by a small number of thresholds, each depending on a single feature. For large, high-dimensional datasets, this provides more information than typical clustering methods.

To make our study concrete, we focus on the $k$-means objective. The goal is to find $k$ centers that approximately minimize the sum of the squared distances between $n$ data points in $\R^d$ and their nearest center~\cite{aggarwal09, aloise2009np, arthur2007k,dasgupta2008hardness, kanungo02, ostrovsky2013effectiveness}. In this context, Dasgupta, Frost, Moshkovitz, and Rashtchian have studied the use of a small threshold tree to specify the cluster assignments, exhibiting the first explainable $k$-means clustering algorithm with provable guarantees~\cite{dasgupta2020explainable}. They propose the Iterative Mistake Minimization (\imm{}) algorithm and prove that it achieves a worst-case $O(k^2)$ approximation to the optimal $k$-means clustering cost. However, the \imm{} algorithm and analysis are limited to trees with exactly $k$ leaves (the same as the number of clusters). They also prove a lower bound showing that an $\Omega(\log k)$ approximation is the best possible when restricted to trees with at most $k$ leaves. 

Our goal is to improve upon two shortcomings of previous work by (i) providing an experimental evaluation of their algorithms and (ii) exploring the impact of using more leaves in the threshold tree. We posit that on real datasets it should be possible to find a nearly-optimal clustering of the dataset. In other words, the existing worst-case analysis may be very pessimistic. Furthermore, we hypothesize that increasing the tree size should lead to monotonically decreasing the clustering cost. 

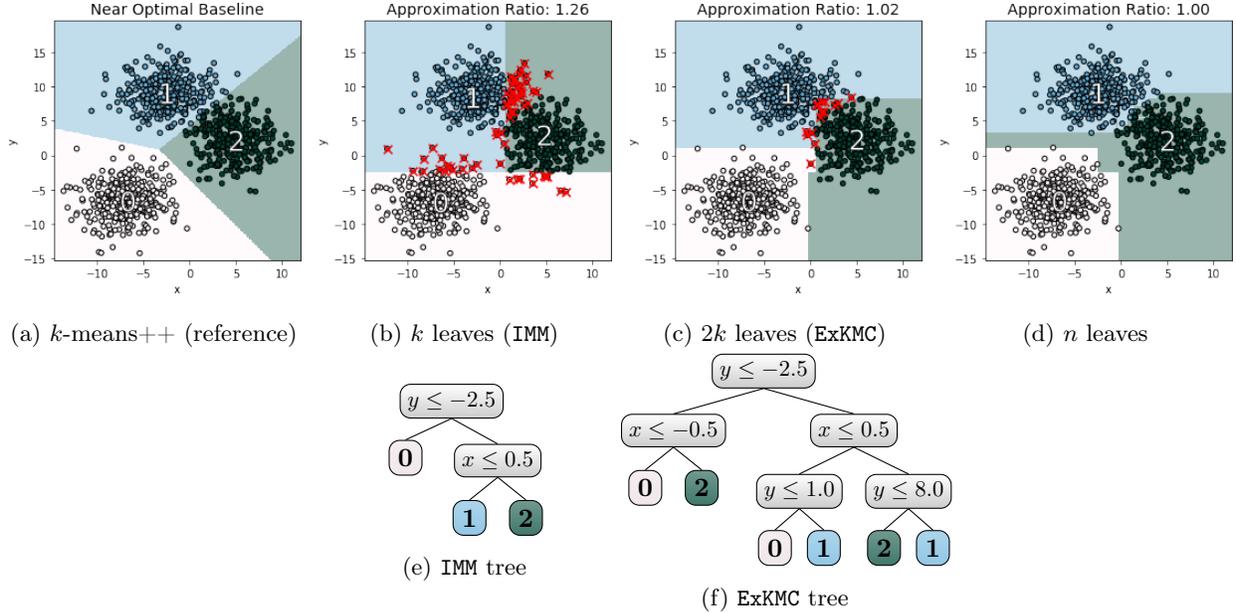
\begin{figure}[tp]
    \centering
    \begin{subfigure}{.25\linewidth}
    \centering
    \includegraphics[width=\linewidth]{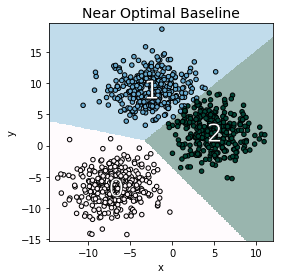}
    \caption{$k$-means++ (reference)}
    \label{fig:boundary_optimal}
    \end{subfigure}%
    \begin{subfigure}{.25\linewidth}
    \centering
    \includegraphics[width=\linewidth]{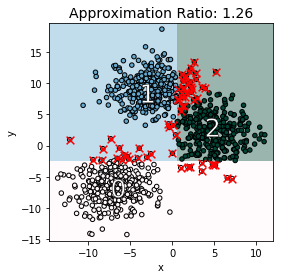}
    \caption{$k$ leaves (\imm{})}
    \label{fig:boundary_k_leaves}
    \end{subfigure}%
    \begin{subfigure}{.25\linewidth}
    \centering
    \includegraphics[width=\linewidth]{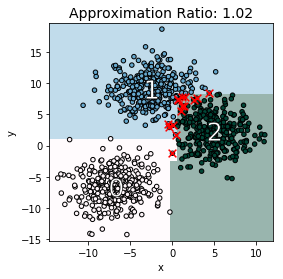}
    \caption{$2k$ leaves (\exkmc{})}
    \label{fig:boundary_2k_leaves}
    \end{subfigure}%
    \begin{subfigure}{.25\linewidth}
    \centering
    \includegraphics[width=\linewidth]{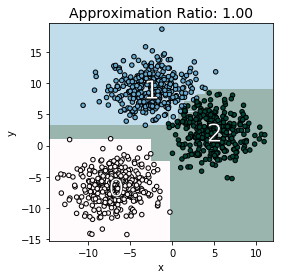}
    \caption{$n$ leaves}
    \label{fig:boundary_n_leaves}
    \end{subfigure}\\
    % \begin{subfigure}{\linewidth}
    % \centering
    % \includegraphics[width=.5\linewidth]{figures/trees_plot.png}
    % \caption{Decision trees}
    % \label{fig:trees}
    % \end{subfigure}\\    
    %%%%%%%%%%%%%%
    %%%%%%%%%%%%
    %%%%%%%%%%%%
    \begin{subfigure}{.25\linewidth}
    \centering
        \begin{tikzpicture}
            [inner/.style={shape=rectangle, rounded corners, draw, align=center, top color=white, bottom color=gray!40, scale=.85},
            leaf/.style={shape=rectangle, rounded corners, draw, align=center, scale=1},
            level 1/.style={sibling distance=4mm},
            level 2/.style={sibling distance=3mm},
            level 3/.style={sibling distance=2mm},
            level 4/.style={sibling distance=2mm},
            level distance=8mm]
            \Tree
            [.\node[inner]{$y \leq -2.5$};
                \node[leaf, top color=Cluster0color!60, bottom color=Cluster0color!80]{\textbf{0}};
                [.\node[inner]{$x \leq 0.5$};
                    \node[leaf, top color=Cluster1color!60, bottom color=Cluster1color!80]{\textbf{1}};
                    \node[leaf, top color=Cluster2color!60, bottom color=Cluster2color!80]{\textbf{2}};
                ]
            ]
        \end{tikzpicture}
    \caption{\imm{} tree}
    \label{fig:tree_k_leaves}
    \end{subfigure}%
    \begin{subfigure}{.25\linewidth}
    \centering
        \begin{tikzpicture}
            [inner/.style={shape=rectangle, rounded corners, draw, align=center, top color=white, bottom color=gray!40, scale=.85},
            leaf/.style={shape=rectangle, rounded corners, draw, align=center, scale=1},
            level 1/.style={sibling distance=4mm},
            level 2/.style={sibling distance=3mm},
            level 3/.style={sibling distance=2mm},
            level 4/.style={sibling distance=2mm},
            level distance=8mm]
            \Tree
            [.\node[inner]{$y \leq -2.5$};
                [.\node[inner]{$x \leq -0.5$};
                    \node[leaf, top color=Cluster0color!60, bottom color=Cluster0color!80]{\textbf{0}};
                    \node[leaf, top color=Cluster2color!60, bottom color=Cluster2color!80]{\textbf{2}};
                ]
                [.\node[inner]{$x \leq 0.5$};
                    [.\node[inner]{$y \leq 1.0$};
                    \node[leaf, top color=Cluster0color!60, bottom color=Cluster0color!80]{\textbf{0}};
                    \node[leaf, top color=Cluster1color!60, bottom color=Cluster1color!80]{\textbf{1}};
                    ]
                    [.\node[inner]{$y \leq 8.0$};
                    \node[leaf, top color=Cluster2color!60, bottom color=Cluster2color!80]{\textbf{2}};
                    \node[leaf, top color=Cluster1color!60, bottom color=Cluster1color!80]{\textbf{1}};
                    ]
                ]
            ]
        \end{tikzpicture}
    \caption{\exkmc{} tree}
    \label{fig:tree_2k_leaves}
    \end{subfigure}%
    \caption{Tree size (explanation complexity) vs. $k$-means clustering quality.}
    \label{fig:boundary_complexity_vs_accuracy}
\end{figure}

As our main contribution, we propose a novel extension of the previous explainable $k$-means algorithm. 
Our method, \exkmc{}, takes as input two parameters $k, k'$ and a set $\mathcal{X} \subseteq \R^d$ with $|\mathcal{X}| = n$. It first builds a threshold tree with $k$ leaves using the \imm{} algorithm~\cite{dasgupta2020explainable}. Then, given a budget of $k' > k$ leaves, it greedily expands the tree to reduce the clustering cost. At each step, the clusters form a refinement of the previous clustering. By adding more thresholds, we gain more flexibility in the data partition, and we also allow multiple leaves to correspond to the same cluster (with $k$ clusters total).

To efficiently determine the assignment of leaves to clusters, we design and analyze a surrogate cost. 
Recall that the \imm{} algorithm first runs a standard $k$-means algorithm, producing a set of $k$ centers that are given as an additional input~\cite{dasgupta2020explainable}. As \exkmc{} expands the tree to $k' >k$ leaves, it minimizes the cost of the current clustering compared to these $k$ reference centers. We assign each leaf a label based on the best reference center, which determines the cluster assignments. The main idea is that by fixing the centers between steps, we can more efficiently determine the next feature-threshold pair to add. We prove that the surrogate cost is non-increasing throughout the execution as the number of leaves $k'$ grows. When $k' = n$, then the $k$-means cost matches that of the reference clustering.

Figure~\ref{fig:boundary_complexity_vs_accuracy} depicts the improvement from using more leaves. The left picture shows a near-optimal $3$-means clustering. Next, the \imm{} algorithm with $k = 3$ leaves leads to a large deviation from the reference clustering. Extending the tree to use $2k = 6$ leaves with \exkmc{} leads to a lower-cost result that better approximates the reference clustering. We form three clusters by subdividing the previous clusters and mapping multiple leaves to the same cluster. Finally, trees with an arbitrary number of leaves can perfectly fit the reference clustering. Figures~\ref{fig:tree_k_leaves} and~\ref{fig:tree_2k_leaves} contains the trees for~\ref{fig:boundary_k_leaves} and~\ref{fig:boundary_2k_leaves} respectively.

To test our new algorithm, we provide an extensive evaluation on many real and synthetic datasets. We show that \exkmc{} often has lower $k$-means cost than several baselines. We also find that prior \imm{} analysis is pessimistic because their algorithm achieves clustering cost within 5--30\% of the cost of standard $k$-means algorithms. Next, we explore the effect of using more than $k$ leaves. On eleven datasets, we show that threshold trees with between $2k$ and $4k$ leaves actually suffice to get within 1--2\% of the cost of a standard $k$-means implementation. The only outlier is CIFAR-10, where we conjecture that pixels are insufficient features to capture the clustering. Overall, we verify that it is possible to find an explainable clustering with high accuracy, while using only $O(k)$ leaves for $k$-means clustering.

\subsection{Related Work} 

We address the challenge of obtaining a low cost $k$-means clustering using a small decision tree. Our approach has roots in previous works on clustering with unsupervised decision trees~\cite{basak2005interpretable,bertsimas2018interpretable, chang2002new, dasgupta2020explainable, de1997using,  fraiman2013interpretable, liu2005clustering, yasami2010novel} and in prior literature on extending decision trees for tasks beyond classification~\cite{geurts2011learning,geurts2007inferring,jernite2017simultaneous,louppe2013understanding, schrynemackers2015classifying,pliakos2018global, ram2011density}.

Besides the \imm{} algorithm~\cite{dasgupta2020explainable}, prior explainable clustering algorithms optimize different objectives than the $k$-means cost, such as the Silhouette metric~\cite{bertsimas2018interpretable}, density measures~\cite{liu2005clustering}, or interpretability scores~\cite{saisubramanian2020balancing}.
Most similar to our approach, a localized version of the 1-means cost has been used for greedily splitting nodes when growing the tree~\cite{fraiman2013interpretable, ghattas2017clustering}. We compare \exkmc{} against two existing methods: \cubt{}~\cite{fraiman2013interpretable} and \cltree{}~\cite{liu2005clustering}. 
We also compare against KDTree~\cite{bentley1975multidimensional} and, after generating the cluster labels, the standard decision tree classification method CART~\cite{breiman1984classification}.

Clustering via trees is explainable by design. We contrast this with the indirect approach of clustering with a neural network and then explaining the network~\cite{kauffmann2019clustering}. A generalization of tree-based clustering has been studied using rectangle-based mixture models~\cite{chen2016interpretable, chen2018phdinterpretable, pelleg2001mixtures}. Their focus differs from ours as they consider including external information, such as prior knowledge, via a graphical model and performing inference with variational methods. Clustering after feature selection~\cite{boutsidis2009unsupervised, cohen2015dimensionality} or feature extraction~\cite{becchetti2019oblivious,boutsidis2014randomized, makarychev2019performance} reduces the number of features, but it does not lead to an explainable solution because it requires running a non-explainable $k$-means algorithm on the reduced feature space. 

\subsection{Our contributions}

We present a new algorithm, \exkmc{}, for explainable $k$-means clustering with the following properties:

\smallpar{Explainability-accuracy trade-off} We provide a simple method to expand a threshold tree with $k$ leaves into a tree with a specified number of leaves. At each step, we aim to better approximate a given reference clustering (such as from a standard $k$-means implementation). The key idea is to minimize a surrogate cost that is based on the reference centers instead of using the $k$-means cost directly. By doing so, we efficiently expand the tree while obtaining a good $k$-means clustering.

\smallpar{Convergence} We demonstrate empirically that \exkmc{} quickly converges to the reference clustering as the number of leaves increases. On many datasets, the cost ratio between our clustering and the reference clustering goes to 1.0 as the number of leaves goes from $k$ to $4k$, where $k$ is the number of labels for classification datasets. In theory, we prove that the surrogate cost is non-increasing throughout the execution of \exkmc{}, verifying that we trade explainability for clustering accuracy.

\smallpar{Low cost} On a dozen datasets, our algorithm often achieves a lower $k$-means cost for a given number of leaves compared to four baselines:  \cubt{}~\cite{fraiman2013interpretable}, \cltree{}~\cite{liu2005clustering}, KDTree~\cite{bentley1975multidimensional}, and \cart{}~\cite{loh2011classification}.

\smallpar{Speed} Using only standard optimizations (e.g., dynamic programming), \exkmc{} can cluster fairly large datasets (e.g., CIFAR-10 or covtype) in under 15 minutes using a single processor, making it a suitable alternative to standard $k$-means in data science pipelines. The main improvement comes from using the surrogate cost, instead of the $k$-means cost, to determine the cluster assignments. 
%In the conclusion we suggest some data-driven heuristics to further improve the running time.

\section{Preliminaries}

We let $[n] = \{1,2,\ldots,n\}$.  For $k \geq 1$, a {\em $k$-clustering} refers to a partition of a dataset into $k$ clusters. 
Let $C^1,\ldots,C^k$ be a $k$-clustering of  $\mathcal{X} \subseteq \R^d$ with $|\mathcal{X}| =n$ and $\vectmu^j = \mathsf{mean}(C^j)$. The {\em $k$-means cost} is $\sum_{j=1}^k \sum_{\vectx \in C^j} \|\vectx - \vectmu^j\|_2^2$. 
It is NP-hard to find the optimal $k$-means  clustering~\cite{aloise2009np, dasgupta2008hardness} or a close approximation~\cite{awasthi2015hardness}.
Standard algorithms for $k$-means are global and iterative, leading to complicated clusters that depend on the data in a hard to explain way~\cite{aggarwal09, arthur2007k,kanungo02, ostrovsky2013effectiveness}.  

\smallpar{Explainable clustering}
Let $T$ be a binary threshold tree with $k' \geq k$ leaves, where each internal node contains a single feature $i \in [d]$ and threshold $\theta \in \R$. 
We also consider a labeling function $\ell:\mathsf{leaves}(T) \to [k]$ that maps the leaves of $T$ to clusters.  The pair $(T,\ell)$ induces a $k$-clustering of $\mathcal{X}$ as follows. First, $\mathcal{X}$ is partitioned via $T$ using the feature-threshold pairs on the root-to-leaf paths. Then, each point $\vectx \in \mathcal{X}$ is assigned to one of $k$ clusters according to how $\ell$ labels its leaf.  Geometrically, the clusters reside in cells bounded by axis-aligned cuts, where the number of cells equals the number of leaves.
This results in a $k$-clustering $\widehat{C}^1,\ldots, \widehat{C}^k$ with means $\widehat{\vectmu}^j = \mathrm{mean}(\widehat{C}^j)$, and we denote the $k$-means cost of the pair $(T,\ell)$ as
$
\cost(T,\ell) = \sum_{j=1}^k \sum_{\vectx \in \widehat{C}^j} \|\vectx - \widehat{\vectmu}^j \|_2^2. 
$

\smallpar{Iterative Mistake Minimization (IMM)}
Previous work has exhibited the \imm{} algorithm that produces a threshold tree with $k$ leaves~\cite{dasgupta2020explainable}. It first runs a standard $k$-means algorithm to find $k$ centers. Then, it iteratively finds the best feature-threshold pair to partition the data into two parts. At each step, the number of mistakes is minimized, where a mistake occurs if a data point is separated from its center. Each partition also enforces that at least one center ends up in both children, so that the tree terminates with exactly $k$ leaves. Each leaf contains one center at the end, and the clusters are assigned based this center. The \imm{} algorithm provides a $O(k^2)$ approximation to the optimal $k$-means cost, assuming that a constant-factor approximation algorithm generates the initial $k$ centers.

We use the \imm{} algorithm to build a tree with $k$ leaves as the initialization for our algorithm. Then, we smoothly trade explainability for accuracy by expanding the tree to have $k' > k$ leaves, for a user-specified parameter $k'$. More formally, we develop an algorithm to solve the following problem:

\smallpar{Problem Formulation} 
Given a dataset $\mathcal{X} \subseteq \R^d$ and parameters $k, k'$ with $k' \geq k$, the goal is to efficiently construct a binary threshold tree $T$ with $k'$ leaves and a function $\ell:\mathsf{leaves}(T) \to [k]$ such that $(T,\ell)$ induces a $k$-clustering of $\mathcal{X}$ with as small $k$-means cost as possible.

\section{Our Algorithm} 
We describe our explainable clustering algorithm, \exkmc{}, that efficiently finds a tree-based $k$-clustering of a dataset. 
Starting with a base tree (either empty or from an existing algorithm like \imm{}), \exkmc{} expands the tree by replacing a leaf node with two new children. In this way, it refines the clustering, while allowing the new children to be mapped to different clusters.  A key optimization is to use a new surrogate cost to determine both the best threshold cut and the labeling of the leaves. At the beginning, we run a standard $k$-means algorithm and generate $k$ reference centers. Then, the surrogate cost is the $k$-means cost if the centers were the reference centers. By fixing the centers, instead of changing them at every step, we determine the cluster label for each leaf independently (via the best reference center). For a parameter $k'$, our algorithm terminates when the tree has $k'$ leaves.
 
\subsection{Surrogate cost}
Our starting point is a set of $k$ reference centers $\vectmu^1,\ldots,{\vectmu^k}$, obtained from a standard $k$-means algorithm. This induces a clustering with low $k$-means cost.
While it is possible to calculate the actual $k$-means cost as we expand the tree, it is difficult and time-consuming to recalculate the distances to a dynamic set of centers. Instead, we fix the $k$ reference centers, and we define the surrogate cost as the sum of squared distances between points and their closest reference center.

% OLD definition... 
% \begin{definition}[Surrogate Cost]
% Given centers $\vectmu^1,\ldots, \vectmu^k$ and a threshold tree $T$ that defines the  clustering $(\widehat{C}^1,\ldots, \widehat{C}^{k'})$, the \emph{surrogate cost} is defined as 
% \begin{align*}
% \surcost^{\vectmu^1,\ldots,\vectmu^k}(T) &= \sum_{j=1}^{k'} \sum_{\vectx \in \widehat{C}^j} \|\vectx - c^{\vectmu_1\ldots, \vectmu_k}(\widehat{C}^j) \|_2^2, 
% \end{align*}
% where $c^{\vectmu_1\ldots, \vectmu_k}(C^j) = \argmin_{i\in[k]}\sum_{\vectx\in C^j}\norm{\vectx-\vectmu^i}$ is the best center for $C^j$ out of centers $\vectmu^1,\ldots,\vectmu^k$.
% \end{definition}
\begin{definition}[Surrogate cost]
Given centers $\vectmu^1,\ldots, \vectmu^k$ and a threshold tree $T$ that defines the  clustering $(\widehat{C}^1,\ldots, \widehat{C}^{k'})$, the \emph{surrogate cost} is defined as 
$$
\surcost^{\vectmu^1,\ldots,\vectmu^k}(T) = \sum_{j=1}^{k'} 
\min_{i\in[k]}
\sum_{\vectx \in 
\widehat{C}^j}
 \norm{\vectx-\vectmu^i}_2^2.
$$
\end{definition}
The difference between the new surrogate cost and the $k$-means cost is that the centers are \emph{fixed}. In contrast, the optimal $k$-means centers are the means of the clusters, and therefore, they would change throughout the execution of the algorithm.  
Before we present our algorithm in detail, we mention that at each step the goal will be to expand the tree by splitting an existing leaf into two children. To see the benefit of the surrogate cost, consider choosing a split at each step that minimizes the actual $k$-means cost of the new clustering. This requires time $\Omega(dkn)$ for each possible split because we must iterate over the entire dataset to calculate the new cost. In Section~\ref{sec:speedup-math}, we provide the detailed argument showing that {\exkmc} only takes time $O(dkn_v)$ at each step, which is an improvement as the number of surviving points $n_v$ in a node often decreases rapidly as the tree grows.

The surrogate cost has several further benefits. The best reference center for a cluster is independent of the other cluster assignments. This independence makes the calculation of $\surcost$ more efficient, as there is no need to recalculate the entire cost if some of the points are added or removed from a cluster. The surrogate cost is also an upper bound on the $k$-means cost. Indeed, it is known that using the mean of a cluster as its center can only improve the $k$-means cost. In Section~\ref{sec:theory}, we provide theoretical guarantees, such as showing that our algorithm always converges to the reference clustering as the number of leaves increases. Finally, in Section~\ref{section:experiments}, we empirically show that minimizing the surrogate cost still leads to a low-cost clustering, nearly matching the reference cost.  We now utilize this idea to design a new algorithm for explainable clustering. 

Algorithm~\ref{algo:exkmc} describes the \exkmc{} algorithm, which uses subroutines in Algorithm~\ref{algo:sub}. It takes as input a value $k$, a dataset $\mathcal{X}$, and a number of leaves $k' \geq k$. The first step will be to generate $k$ reference centers $\vectmu^1,\ldots,\vectmu^k$ from a standard $k$-means implementation and to build a threshold tree $T$ with $k$ leaves (for evaluation, $T$ is the output of the \imm{} algorithm). For simplicity, we refer to these as inputs as well. \exkmc{} outputs a tree $T'$ and a labeling $\ell:\mathsf{leaves}(T') \to [k]$ that assigns leaves to clusters. Notably, the clustering induced by $(T',\ell)$ always refines the one from $T$. 

\begin{table}[!htb]
    \centering
    \begin{tabular}[t]{ll}
    \centering
    \begin{minipage}[t]{.5\textwidth}
    \begin{algorithm}[H]
    \SetKwFunction{EX-Means}{EX-Means}
    \SetKwFunction{ExpandTree}{ExpandTree}
    \SetKwInOut{Input}{Input}\SetKwInOut{Output}{Output}\SetKwInOut{Preprocess}{Preprocess}
    \Input{
    	$\mathcal{X}$ -- Set of vectors in $\mathbb{R}^d$\\
    	$\centersSet$ -- Set of $k$ reference centers\\
    	$T$ -- Base tree\\
    	$k'$ -- Number of leaves\\
      }
    \Output{
        Labeled tree with $k'$ leaves
    }
    
    \LinesNumbered
    \setcounter{AlgoLine}{0}
    \BlankLine
    
    $\xvar{splits} \leftarrow \xfunc{dict}()$\;
    $\xvar{gains} \leftarrow \xfunc{dict}()$\;
    \ForEach{$\xvar{leaf} \in T.leaves$}
    {\label{alg:init}
        $\xfunc{add\_gain}(\xvar{leaf}, \mathcal{X}, \centersSet, \xvar{splits}, \xvar{gains})$\;
    }
    \While {$\lvert T.leaves \rvert < k'$}
    {
        $\xvar{leaf} \leftarrow \argmax_{\xvar{leaf}} \xvar{gains}[\xvar{leaf}]$\;\label{alg:max_gain}
        $i, \theta \leftarrow \xvar{splits}[\xvar{leaf}]$\;
        $\vectmu^L,\vectmu^R\leftarrow\xfunc{find\_labels}(\cX,\centersSet,i,\theta)$\;
        $\xvar{leaf}.condition \leftarrow ``x_i \leq \theta"$\;
        $\xvar{leaf}.l \leftarrow \xfunc{new Leaf}(label=\vectmu^L)$\;
        $\xvar{leaf}.r \leftarrow \xfunc{new Leaf}(label=\vectmu^R)$\;
        $\xfunc{add\_gain}(\xvar{leaf}.l, \mathcal{X}, \centersSet, \xvar{splits}, \xvar{gains})$\;\label{alg:add_gain}
        $\xfunc{add\_gain}(\xvar{leaf}.r, \mathcal{X}, \centersSet, \xvar{splits}, \xvar{gains})$\;
        $\xfunc{delete}(\xvar{splits}[\xvar{leaf}], \xvar{gains}[\xvar{leaf}])$\;\label{alg:delete_gain}
    }
    \Return $T$\;
    
    \caption{\textsc{\exkmc{}: Expanding \newline Explainable $k$-Means Clustering}}
    \label{algo:exkmc}
    \end{algorithm}
    \end{minipage}
    &
    \begin{minipage}[t]{.5\textwidth}
    \begin{algorithm}[H]
    \SetKwProg{leafGain}{$\xfunc{add\_gain}(\mathsf{leaf},\mathcal{X}, \centersSet, \mathsf{splits}, \mathsf{gains})$:}{}{}
    \LinesNumbered
    \setcounter{AlgoLine}{0}
    \BlankLine
    \leafGain{}{
        $\mathcal{X}_l \leftarrow \{\vectx \in \mathcal{X} \mid \vectx \text{ path ends in } \xvar{leaf}\}$\;
        $i, \theta \leftarrow \argmin_{i,\theta} \xfunc{split\_cost}(\mathcal{X}_l, \centersSet, i, \theta)$\;\label{ln:k_dynamic}
        $\xvar{best\_cost} \leftarrow \xfunc{split\_cost}(\mathcal{X}_l, \centersSet, i, \theta)$\;
        $\xvar{splits}[\xvar{leaf}] \leftarrow (i, \theta)$\;
        $\xvar{gains}[\xvar{leaf}] \leftarrow \surcost_{\centersSet}(\mathcal{X}_l) - \xvar{best\_cost}$\;
        
    }
    
    \SetKwProg{splitCost}{$\xfunc{split\_cost}(\mathcal{X}, \centersSet, i, \theta)$:}{}{}
    \LinesNumbered
    \setcounter{AlgoLine}{0}
    \BlankLine
    \splitCost{}{
    $\mathcal{X}_L \leftarrow \{\vectx \in \mathcal{X} \mid x_i \leq \theta\}$\;
    $\mathcal{X}_R \leftarrow \{\vectx \in \mathcal{X} \mid x_i > \theta\}$\;
    \Return $\surcost_{\centersSet}(\mathcal{X}_L) + \surcost_{\centersSet}(\mathcal{X}_R)$\;
    }
    
    \SetKwProg{findLabels}{$\xfunc{find\_labels}(\mathcal{X}, \centersSet, i, \theta)$:}{}{}
    \LinesNumbered
    \setcounter{AlgoLine}{0}
    \BlankLine
    \findLabels{}{
    $\vectmu^L \leftarrow \argmin_{\vectmu \in \centersSet}\sum_{\vectx\in\cX:x_i\leq\theta}\norm{\vectx-\vectmu}_2^2$\;
    $\vectmu^R \leftarrow \argmin_{\vectmu \in \centersSet}\sum_{\vectx\in\cX:x_i>\theta}\norm{\vectx-\vectmu}_2^2$\;
    \Return $\vectmu^L,\vectmu^R$\;
    }
    
    \BlankLine
    \BlankLine
    \BlankLine
    \BlankLine    
    \BlankLine
    \BlankLine    
    \BlankLine
    \BlankLine
    \BlankLine
\BlankLine\BlankLine\BlankLine\BlankLine\BlankLine\BlankLine

    \caption{\textsc{Subroutines} \newline}
    \label{algo:sub}
    \end{algorithm}
    \end{minipage}\vspace{-2ex}
    \\
    \end{tabular}
\end{table}

\smallpar{Initialization} 
In line \ref{alg:init}, we first compute and store the gain of the $k$ leaves in $T$. The {\em gain} of a split is the difference between the cost with split and without. The details are in the subroutine \texttt{add\_gain}. It stores the best feature-threshold pair and cost improvement in \texttt{splits} and \texttt{gains}, respectively.

\smallpar{Growing the tree} 
We expand the tree by splitting the node with largest gain in Line~\ref{alg:max_gain}. We use best reference center from $\centersSet$ for each of its two children, using the subroutine \texttt{find\_labels}. In Lines~\ref{alg:add_gain}--\ref{alg:delete_gain}, we update the lists \texttt{splits} and \texttt{gains}. At the final step, we create a tree $T'$ with $k'$ labeled leaves, where the implicit labeling function $\ell$ maps a leaf to the lowest cost reference center.

\subsection{Speeding up \exkmc{}}
\label{sec:speedup-math}

The running time is dominated by finding the best split at each step. Fortunately, this can executed in $O(dkn_v + dn_v \log n_v)$ time where $n_v$ is the number of points surviving at node $v$.  The term $O(n_v \log n_v)$ comes from sorting the points in each coordinate. Then, we go over all splits $(i,\theta)$ and find the one that minimizes $\sum_{\vectx\in C^L}\norm{\vectx-\vectmu^L}_2^2 + \sum_{\vectx\in C^R}\norm{\vectx-\vectmu^R}_2^2$, where $C^L$ contains all points in $v$ where $x_i\leq\theta$ and $\vectmu^L$ is the best center for $C^L$ among the $k$ reference centers (cluster $C^R$ and center $\vectmu^R$ are defined similarly for points $x_i > \theta$). At first glance, it seems that the running time is $O(d^2kn_v^2)$, where $dn_v$ is the number of possible splits, $k$ is the possible value of $\vectmu^L$ (and $\vectmu^R$), and $O(dn_v)$ is the time to calculate the cost of each split and center value. To improve the running time we can rewrite the cost as $$\sum_{\vectx \in C^L \cup C^R}\norm{\vectx}_2^2-2\inner{\sum_{\vectx\in C^L}\vectx}{\vectmu^L}+|C^L|\norm{\vectmu^L}_2^2-2\inner{\sum_{\vectx\in C^R}\vectx}{\vectmu^R}+|C^R|\norm{\vectmu^R}_2^2.$$ Since we care about minimizing the last expression, we can ignore the term $\sum_{\vectx \in C^L \cup C^R}\norm{\vectx}_2^2$ as it is independent of the identity of the split. Using dynamic programming, we go over all splits and save $\sum_{\vectx\in C^L}\vectx$ and $\sum_{\vectx\in C^R}\vectx$. An update then only takes $O(dk)$ time, reducing the total time to $O(d^2kn_v + dn_v \log n_v).$ When the dimensionality $d$ is large, we can use an additional improvement by saving $\inner{\vectx}{\vectmu}$ for each $\vectx$ and $\vectmu$, reducing the total running time $O(dkn_v + dn_v \log n_v).$

\subsection{Theoretical guarantees}
\label{sec:theory}
Now that we have defined our {\exkmc} algorithm, we provide some guarantees on its performance (we defer all proofs for this section to Appendix~\ref{appendix:proofs}). We first prove two theorems about the surrogate cost, showing that it satisfies the two desirable properties of being non-increasing and also eventually converging to the reference clustering. Next, we verify that \exkmc{} has a worst-case approximation ratio of $O(k^2)$ compared to the optimal $k$-means clustering when using \imm{} to build the base tree. Finally, we provide a separation between \imm{} and \exkmc{} for a difficult dataset.

\begin{theorem}\label{thm:surrogate-non-increasing}
The surrogate cost, $\surcost$, is non-increasing throughout the execution of {\em \exkmc}. 
\end{theorem}

To prove the theorem notice that when the algorithm performs a split at node $v$, the cost of all points not in $v$ will remain intact. Additionally, the algorithm can assign the two children of $v$ the same label as in $v$. This choice will not change the cost of the points in $v.$ The full proof appears in the appendix. 
Eventually \exkmc{} converges to a low-cost tree, as the next theorem proves. Specifically, after $n$ steps in the worst-case, we will always get a refinement, see Corollary~\ref{cor:surrogate_cost_after_n_stpss} in Appendix~\ref{appendix:proofs}.
\begin{theorem}\label{thm:algorithm_refinement}
Let $C$ be a reference clustering. If while running {\em \exkmc} the threshold tree $T$ refines $C$, then the clustering induced by $(T,\ell)$ equals $C$, where $\ell$ is the surrogate cost labeling.
\end{theorem}

We also provide a worst-case guarantee compared to the optimal $k$-means solution. We show that the \exkmc{} algorithm using the \imm{} base tree provably achieves an $O(k^2)$ approximation to the $k$-means cost. The proof uses Theorem~\ref{thm:surrogate-non-increasing} combined with the previous analysis of the \imm{} algorithm~\cite{dasgupta2020explainable}, see Appendix~\ref{appendix:proofs}.
\begin{theorem}\label{thm:worst_case}
Algorithm~\ref{algo:exkmc} with the {\em \imm{}} base tree is an $O(k^2)$ approximation to the optimal $k$-means cost.
\end{theorem}

\subsection{Example where \exkmc{} provably improves upon \imm{}}
Prior work designed a dataset and proved that a tree with $k$ leaves incurs an $\Omega(\log k)$ approximation on this dataset~\cite{dasgupta2020explainable}, which we call {\em Synthetic II}. It contains $k$ clusters that globally are very far from each other, but locally, clusters look the same. 

\paragraph{Synthetic II dataset.}
This dataset $\cD \subseteq \{-1,0,1\}^d$ consists of $k$ clusters where any two clusters are very far from each other while inside any cluster the points differ by at most two features. 
The dataset is created by first taking $k$ random binary \emph{codewords} $\vectv^1,\ldots, \vectv^k\in\{-1,1\}^d$. The distance between any two codewords is at least $\alpha d,$ where $\alpha$ is some  constant.
For $i \in [k]$, we construct cluster $C^{\vectv^i}$ by taking the codeword $\vectv^i$ and then changing one feature at a time to be equal to $0.$ The size of each cluster is $|C^{\vectv^i}|=d.$ 
There are in total $dk$ points in the dataset $\cD$. The optimal $k$ centers for this dataset are the codewords. We take $d > c k^2$ for some universal constant $c > 1$, and we assume that $k$ is sufficiently large.

Our approach of explaining this dataset contains a few steps (i) using $k$-means++ to generate the reference centers, (ii) building the base tree with \imm{}, and (iii) expanding the tree with \exkmc{}. We prove that this leads to an optimal tree-based clustering with $O(k \log k)$ leaves. In fact, the cost ratio decreases linearly as the number of leaves increases, as the next theorem proves. We experimentally confirm this linear decrease in the next section. 

\begin{theorem}\label{thm:synthetic_II_upper_bound}
Let $\cD$ be the dataset described above. There is a constant $c'>0$ such that with probability at least $1-O\left(\frac{k\log k}{d}\right)$ over the randomness in $k$-means++, 
the output of {\em \exkmc} with centers from $k$-means++, using the base tree from {\em \imm{}} and $k'$ desired centers, has approximation ratio at most $$O\left(\max\left\{\log k -c'\cdot\frac{k'}{k},1\right\}\right)$$
compared to the optimal $k$-means clustering of $\cD$.
\end{theorem}
We show that initially the \imm{} algorithm has approximation ratio of $O(\log k)$. Interestingly, this is also the lower bound for trees with exactly $k$ leaves, and hence, it constructs the best possible tree up to a constant factor. The theorem also states that after $O(k\log k)$ steps we reach the optimal clustering. We state this observation as the following corollary. 
\begin{corollary}
With probability at least $1-O\left(\frac{k\log k}{d}\right)$ over the randomness in $k$-means++, 
the output of {\em \exkmc}, with centers from $k$-means++ and base tree of {\em \imm{}}, after $O(k\log k)$ steps of the algorithm, it reaches the optimal $k$-means clustering of $\cD$. 
\end{corollary}

The full proof of Theorem~\ref{thm:synthetic_II_upper_bound} is in the appendix. The main steps of the proof are the following. We first prove that with high probability the output of the $k$-means++ algorithm is actually the optimal set of $k$ centers, that is, the codewords in the dataset construction. Intuitively, this holds because globally the optimal centers and clusters are very far from each other.  The second step is to define points that are \emph{mistakes}, which are points in a leaf $u$ with the property that their optimal center is different than the label of $u.$ To guarantee that the algorithm always makes progress, we prove that there is always a split that reduces the number of mistakes. Lastly, we prove that each such split reduces the cost by $\Omega(d).$ This completes the proof.

\section{Empirical Evaluation}
\label{section:experiments}
% \smallpar{Experimental Set-up} Talk about the basic algorithms and baselines. Also briefly mention the datasets that we use. Most of the details will have to go in the appendix. 
\smallpar{Algorithms}
We compare the following clustering methods (see Appendix~\ref{sec:appendix_setup} for details):
\begin{itemize}
    \item {\bf Reference Clustering.} \sklearn{} \texttt{KMeans}, 10 random initializations, 300 iterations. 
    \item {\bf CART.} Standard decision tree from \sklearn{}, minimizing \texttt{gini} impurity. Points in the dataset are assigned labels using the reference clustering.
    \item {\bf KDTree.} Split highest variance feature at median threshold. Size determined by \texttt{leaf\_size} parameter. Label leaves to minimize $\surcost$ w.r.t. centers of the reference clustering.
    \item {\bf CLTree.} Explainable clustering method. Public implementation \cite{cltreeP, liu2005clustering}.
    \item {\bf CUBT.} Explainable clustering method. Public implementation~\cite{cubtR, fraiman2013interpretable}.
    \item {\bf ExKMC.} Our method (Algorithm~\ref{algo:exkmc}) starting with an empty tree; minimizes $\surcost$ at each split w.r.t. centers of the reference clustering.
    \item {\bf ExKMC (base: IMM).} Our method (Algorithm~\ref{algo:exkmc}) starting with an \imm{} tree with $k$ leaves; then, minimizes $\surcost$ at each split w.r.t. centers of the reference clustering.
\end{itemize}

\smallpar{Set-up} 
We use 10 real and 2 synthetic datasets; details in Appendix~\ref{sec:appendix_datasets}. The number of clusters $k$ and the number of leaves $k'$ are inputs. We start with $k$ equal to number of labels for classification datasets. We plot the cost ratio compared to the reference clustering (best $=1.0$). For the baselines, we do hyperparameter tuning and choose the lowest cost clustering at each $k'$. \cubt{} and \cltree{} could only be feasibly executed on six small real datasets (we restrict computation time to one hour).

\begin{figure}
        \renewcommand{\arraystretch}{1.2}
%\begin{table}
    \begin{tabular}{c c c c}       
        \rowcolor{gray!15}\multicolumn{4}{l}{\textbf{Small Datasets}}\\
        \begin{subfigure}{.24\linewidth}
        \centering
        \includegraphics[width=\linewidth]{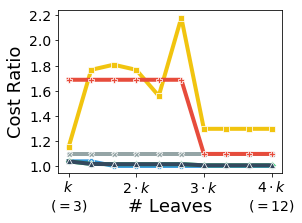}
        \caption{\scriptsize{Iris}}
        \label{fig:iris}
        \end{subfigure}& 
        \begin{subfigure}{.24\linewidth}
        \centering
        \includegraphics[width=\linewidth]{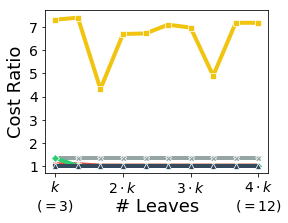}
        \caption{\scriptsize{Wine}}
        \label{fig:wine}
        \end{subfigure}& 
        \multicolumn{2}{c}{
         \begin{subfigure}{.4\linewidth}
            \centering
            \includegraphics[width=\linewidth]{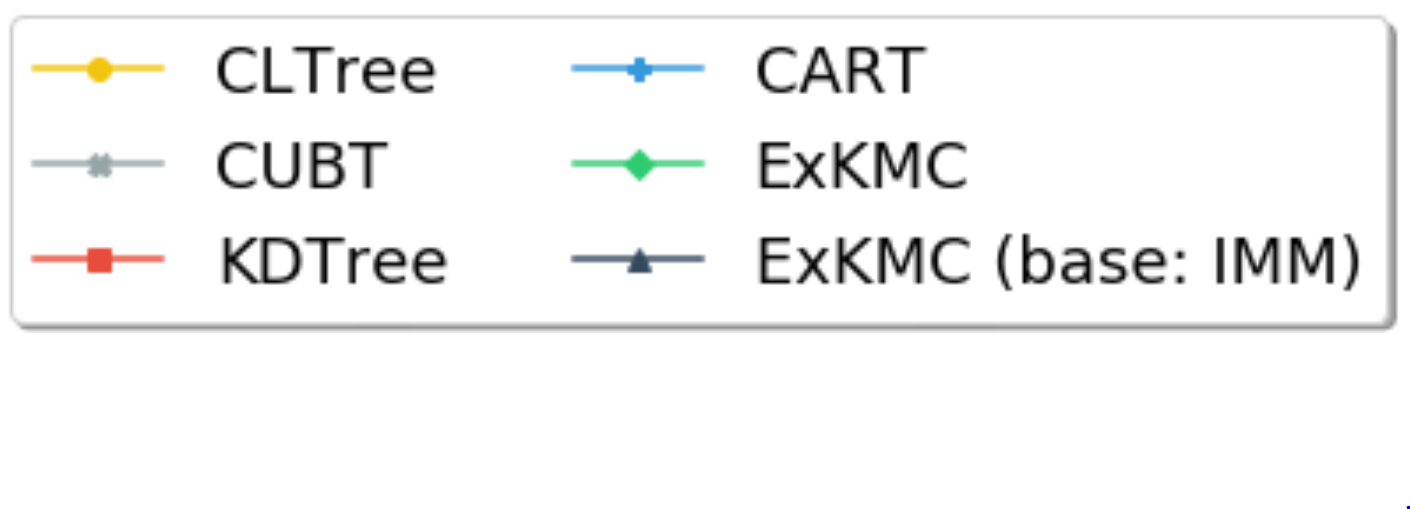}
        \end{subfigure}
        }\\
        \begin{subfigure}{.24\linewidth}
        \centering
        \includegraphics[width=\linewidth]{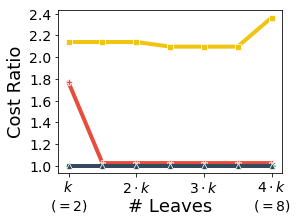}
        \caption{\scriptsize{Breast Cancer}}
        \label{fig:brease_cancer}
        \end{subfigure}& 
        \begin{subfigure}{.24\linewidth}
        \centering
        \includegraphics[width=\linewidth]{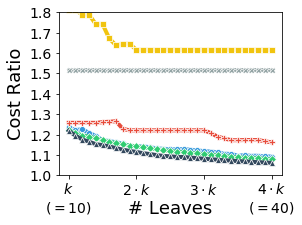}
        \caption{\scriptsize{Digits}}
        \label{fig:digits}
        \end{subfigure}&
        \begin{subfigure}{.24\linewidth}
        \centering
        \includegraphics[width=\linewidth]{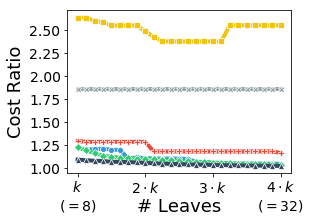}
        \caption{\scriptsize{Mice Protein}}
        \label{fig:mice_protein}
        \end{subfigure}&
        \begin{subfigure}{.24\linewidth}
        \centering
        \includegraphics[width=\linewidth]{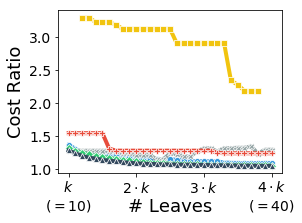}
        \caption{\scriptsize{Anuran}}
        \label{fig:anurn}
        \end{subfigure} \vspace{1ex}\\
        %\midrule
        \rowcolor{gray!15} \multicolumn{4}{l}{\textbf{Larger Datasets}}\\
        \begin{subfigure}{.24\linewidth}
        \centering 
        \includegraphics[width=\linewidth]{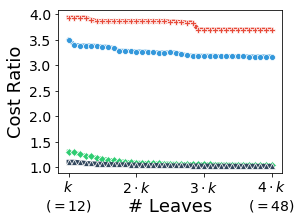}
        \caption{\scriptsize{Avila}}
        \label{fig:avila}
        \end{subfigure}&
        \begin{subfigure}{.24\linewidth}
        \centering
        \includegraphics[width=\linewidth]{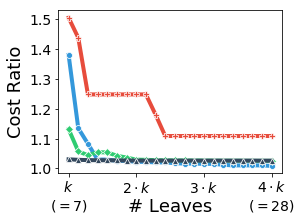}
        \caption{\scriptsize{Covtype}}
        \label{fig:covtype}
        \end{subfigure}&
        \begin{subfigure}{.24\linewidth}
        \centering
        \includegraphics[width=\linewidth]{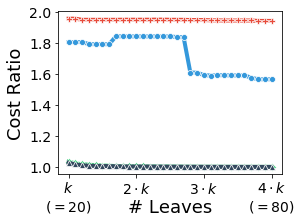}
        \caption{\scriptsize{20newsgroups}}
        \label{fig:20newsgroups}
        \end{subfigure}&
        \begin{subfigure}{.24\linewidth}
        \centering
        \includegraphics[width=\linewidth]{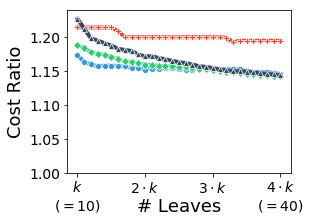}
        \caption{\scriptsize{CIFAR-10}}
        \label{fig:cifar10}
        \end{subfigure}\vspace{1ex} \\
        %\midrule
        \rowcolor{gray!15}
        \multicolumn{4}{l}{\textbf{Synthetic Datasets}}\\
        &
        \begin{subfigure}{.24\linewidth}
        \centering
        \includegraphics[width=\linewidth]{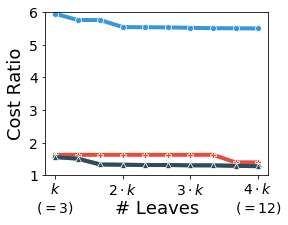}
        \caption{\scriptsize{Synthetic I}}
        \label{fig:synthetic1}
        \end{subfigure}&
        \begin{subfigure}{.24\linewidth}
        \centering
        \includegraphics[width=\linewidth]{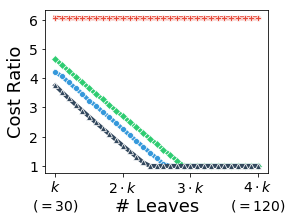}
        \caption{\scriptsize{Synthetic II}}
        \label{fig:synthetic2}
        \end{subfigure}&\\
    \end{tabular}
    \caption{
        % Comparing several methods on ten real and two synthetic datasets. 
        Each graph plots the ratio ($y$-axis) of the tree-based clustering cost to the near-optimal $k$-means clustering as a function of the number of leaves ($x$-axis). Lower is better with best $= 1.0$. Our algorithm (black line) consistently performs well. See Figure~\ref{fig:experiments_error_bars} in Appendix~\ref{sec:appendix_extra_exp} for error bars.
    }
    \label{fig:experiments}
\end{figure}

\subsection{Experimental Results}

\smallpar{Real datasets} \cltree{} performs the worst in most cases, and the cost is often much larger than the other methods. Turning to \cubt{}, we see that on most datasets it is competitive (but often not the best). However, on Digits and Mice Protein, \cubt{} fails to converge to a good clustering. We see that \cart{} performs well on many of the datasets, as expected. On Avila and 20newsgroups, \cart{} has a fairly high cost. For the small datasets, KDTree performs competitively, but on large datasets, the cost remains high. On all of the datasets, our method \exkmc{} performs well. It often has the lowest cost throughout, except for CIFAR-10, where it requires around $3.5 \cdot k$ leaves to be competitive. When the number of leaves is exactly $k$, we can also see the performance of the \imm{} algorithm, where we see that its cost is quite low, in contrast to the previous theoretical analysis~\cite{dasgupta2020explainable}. As usual, the outlier is CIFAR-10, where the cost is very high at $k$ leaves, and then quickly decreases as \exkmc{} expands the tree. We also evaluate \exkmc{} by starting with an empty tree (without using \imm{} at all). In general, the cost is worse or the same as \exkmc{} with \imm{}.
We observe that the \cltree{} cost varies as a function of the number of leaves. The reason is that we perform a hyperparameter search for each instance separately, and perhaps surprisingly, the cost can sometimes increase. While we could have used the best tree with fewer leaves, this may be less representative of real usage. Indeed, the user would specify a desired number of leaves, and they may not realize that fewer leaves would lead to lower $k$-means cost.

\smallpar{Synthetic datasets}
We highlight two other aspects of explainable clustering. The Synthetic I dataset in Figure \ref{fig:synthetic1} is bad for \cart{}. We adapt a dataset from prior work that forces \cart{} to incur a large cost due to well-placed outliers~\cite{dasgupta2020explainable}. \cart{} has cost ratio above five, while other methods converge to a near-optimal cost. The Synthetic II dataset in Figure~\ref{fig:synthetic2} is sampled from a hard distribution that demonstrates an $\Omega(\log k)$ lower bound for any explainable clustering with $k$ leaves~\cite{dasgupta2020explainable}. The $k$ centers are random $\pm 1$ vectors in $\R^d$, and the clusters contain the $d$ vectors where one of the center's coordinates is set to zero. As expected, the \imm{} algorithm starts with a high $k$-means cost. When \exkmc{} expands the tree to have more leaves, the cost quickly goes down (proof in Appendix~\ref{apx:synthetic_II}).

\smallpar{Distance to the reference clustering} 
% So far, we have focused on clustering cost. 
In Appendix~\ref{sec:appendix_extra_exp}, we report accuracy when points are labeled by cluster. Overall, we see the same trends as with cost. 
However, on some datasets, \cart{} more closely matches the reference clustering than \exkmc{}, but the \cart{} clustering has a higher cost. This highlights the importance of optimizing the $k$-means cost, not the classification accuracy.

\smallpar{Qualitative analysis}
Figure~\ref{fig:imm-exkmc-tree} depicts two example trees with four and eight leaves, respectively. on a subset of four clusters from the 20newsgroups dataset. The \imm{} base tree uses three features (words) to define four clusters. Then, \exkmc{} expands one of the leaves into a larger subtree, using seven total words to construct more nuanced clusters that better correlate with the newsgroup topics.

\smallpar{Running time}
Figure \ref{fig:runtime_single_process} shows the runtime of three methods on seven real datasets (commodity laptop, single process, i7 CPU @ 2.80GHz, 16GB RAM). Both \imm{} and \exkmc{} first run \texttt{KMeans} (from \sklearn{}, 10 initializations, 300 iterations), and we report cumulative times. The explainable algorithms construct trees in under 15 minutes. On six datasets, they incur $0.25\times$ to $1.5\times$ overhead compared to standard \texttt{KMeans}. The 20newsgroups dataset has the largest overhead because \sklearn{} optimizes for sparse vectors while \imm{} and \exkmc{} currently do not. In Appendix~\ref{sec:appendix_runtime}, we also measure an improvement to \exkmc{} with feature-wise parallelization of the expansion step.

\smallpar{Surrogate vs. actual cost} In Figure~\ref{fig:suroogate_vs_optimal}, we compare the surrogate cost (using the $k$ reference centers) with the actual $k$-means cost (using the means of the clusters as centers). In general, the two costs differ by at most 5-10\%. On three datasets, they converge to match each other as the number of leaves grows. Covtype is an exception, where the two costs remain apart. 

\begin{figure}[!htb]
    \centering
    \hspace*{\fill}%
    \begin{minipage}[t]{.65\textwidth}
        \centering
        \vspace{0pt}
        \includegraphics[width=\linewidth]{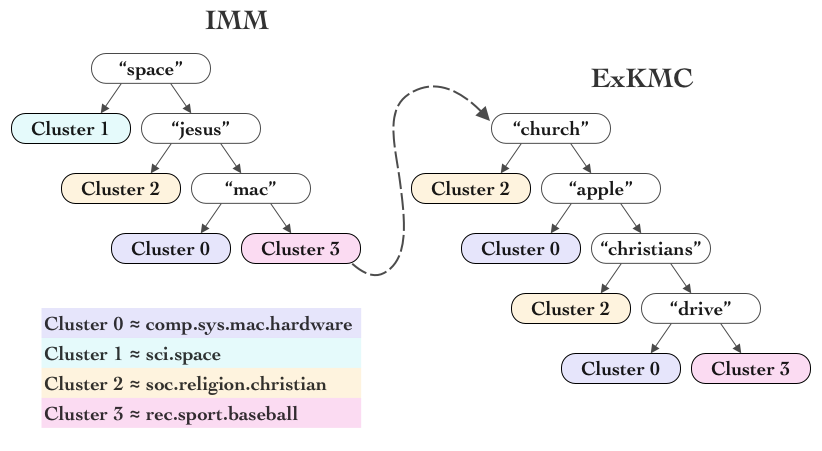}
        \caption{{Explainability-accuracy trade-off} using more leaves on a subset of four clusters from 20newsgroups. As \exkmc{} expands the \imm{} tree, it refines the clusters and reduces the $4$-means cost.}
        \label{fig:imm-exkmc-tree}
    \end{minipage}%
    \hfill
    \begin{minipage}[t]{.32\textwidth}
        \centering
        \vspace{0pt}
        \includegraphics[width=\linewidth]{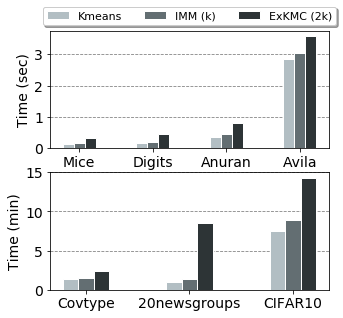}
        \caption{Runtime.}
        \label{fig:runtime_single_process}
    \end{minipage}\\
    \hspace*{\fill}%
\end{figure}
\begin{figure}[!ht]
    \centering
    \begin{subfigure}{.24\linewidth}
    \centering
    \includegraphics[width=\linewidth]{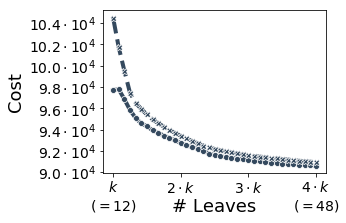}
    \caption{\scriptsize{Avila}}
    \label{fig:avila_sur}
    \end{subfigure}%
    \begin{subfigure}{.24\linewidth}
    \centering
    \includegraphics[width=\linewidth]{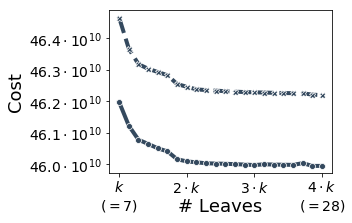}
    \caption{\scriptsize{Covtype}}
    \label{fig:covtype_sur}
    \end{subfigure}%
    \begin{subfigure}{.24\linewidth}
    \centering
    \includegraphics[width=\linewidth]{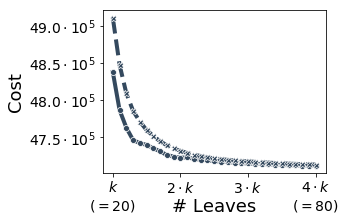}
    \caption{\scriptsize{20newsgroups}}
    \label{fig:20newsgroups_sur}
    \end{subfigure}%
    \begin{subfigure}{.24\linewidth}
    \centering
    \includegraphics[width=\linewidth]{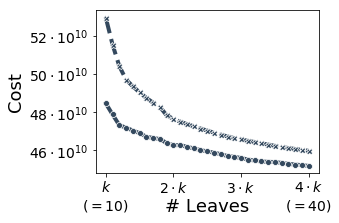}
    \caption{\scriptsize{CIFAR-10}}
    \label{fig:cifar10_sur}
    \end{subfigure}%
    \caption{Comparison of $\surcost$ (dashed line) vs. $\cost$ (full line) of \exkmc{} with \imm{} base tree.}
    \label{fig:suroogate_vs_optimal}
\end{figure}

% \begin{figure}[!htb]
%     \begin{subfigure}{.22\linewidth}
%     \centering
%     \includegraphics[width=\linewidth]{figures/avila_accuracy_plot.png}
%     \caption{\scriptsize{Avila}}
%     \label{fig:avila_acc}
%     \end{subfigure}%
%     \begin{subfigure}{.22\linewidth}
%     \centering
%     \includegraphics[width=\linewidth]{figures/covtype_accuracy_plot.png}
%     \caption{\scriptsize{Covtype}}
%     \label{fig:covtype_acc}
%     \end{subfigure}%
%     \begin{subfigure}{.22\linewidth}
%     \centering
%     \includegraphics[width=\linewidth]{figures/20newsgroups_accuracy_plot.png}
%     \caption{\scriptsize{20newsgroups}}
%     \label{fig:20newsgroups_acc}
%     \end{subfigure}%
%     \begin{subfigure}{.22\linewidth}
%     \centering
%     \includegraphics[width=\linewidth]{figures/cifar10_accuracy_plot.png}
%     \caption{\scriptsize{CIFAR-10}}
%     \label{fig:cifar10_acc}
%     \end{subfigure}%
%     \caption{Convergence to reference centers}
%     \label{fig:accuracy}
% \end{figure}

\subsection{Discussion}

\smallpar{Trade-off} The main objective of our new algorithm is to provide a flexible trade-off between explainability and accuracy. Compared to the previous \imm{} algorithm, we see that using \exkmc{} to expand the threshold tree consistently leads to a lower cost clustering. We also see that using our surrogate cost improves the running time without sacrificing effectiveness. 

\smallpar{Convergence} 
On some datasets, \imm{} produces a fairly low $k$-means cost with $k$ leaves. Expanding the tree with \exkmc{} to  between $2k$ and $4k$ leaves often results in nearly the same cost as the reference clustering. CIFAR-10 is an outlier, where none of the methods converge when using pixels as features (see also Appendix~\ref{sec:appendix_extra_exp}). In practice, a tree with $4k$ leaves only slightly increases the explanation complexity compared to a tree with $k$ leaves (and $k$ leaves are necessary). \exkmc{} successfully achieves a good tree-based clustering with better interpretability than standard $k$-means methods.

\smallpar{Low cost} The most striking part of the experiments is that 
tree-based clusterings can match the cost of standard clusterings. This is possible with a small tree, even on large, high-dimensional datasets. Prior to our work, this was not known to be the case. Therefore, \exkmc{} demonstrates that explanability can be obtained in conjunction with low cost on many datasets.
\section{Conclusion}

We exhibited a new algorithm, 
\exkmc{}, for generating an explainable $k$-means clustering using a threshold tree with a specified number of leaves. Theoretically, our algorithm has the property that as the number of leaves increases, the cost is non-increasing. This enables a deliberate trade-off between explainability and clustering cost. Extensive experiments showed that our algorithm often achieves lower $k$-means cost compared to several baselines. Moreover, we saw that \exkmc{} usually matches the reference clustering cost with only $4k$ leaves. Overall, \exkmc{} efficiently produces explainable clusters, and it could potentially replace standard $k$-means implementations in data science pipelines.  For future work, it would be interesting to prove convergence guarantees for \exkmc{} either based on separation properties of real data or a Gaussian mixture model assumption. Another direction is to reduce the running time through better parallelism and sparse feature vector optimizations.
Finally, our algorithm is feature-based, and therefore, if the features or data have biases, then these biases may also propagate and lead to biased clusters. It would be interesting future work to develop an explainable clustering algorithm that enforces fairness either in the feature selection process or in the composition of the clusters (for examples of such methods, see~\cite{backurs2019scalable, bera2019fair,huang2019coresets,kleindessner2019fair, mahabadi2020individual, schmidt2019fair}). 

\section*{Acknowledgements}
We thank Sanjoy Dasgupta for helpful discussions. We also thank Alyshia Olsen for help designing the figures.
Nave Frost has been funded by the  European Research  Council (ERC) under the European Unions Horizon 2020 research and innovation programme (Grant agreement No. 804302). The contribution of Nave Frost is part of a Ph.D. thesis research conducted at Tel Aviv University.

\medskip

{\small
\bibliography{references}
\bibliographystyle{abbrv}
}

\newpage
\appendix
\section{Omitted proofs and extra theoretical results}
\label{appendix:proofs}
\begin{proof}[Proof of Theorem~\ref{thm:surrogate-non-increasing}.]
Denote by $T^{k'}$ the tree at iteration $k'$ of the algorithm. 
We want to show that
$$\surcost^{\vectmu^1,\ldots,\vectmu^k}(T^{k'+1})\leq \surcost^{\vectmu^1,\ldots,\vectmu^k}(T^{k'}).$$
We prove a stronger claim: for any possible split (not just for the one by \exkmc{}) the surrogate cost will not increase.
Fix a cluster $\widehat{C}^{r}$ in $T^{k'}$  and suppose that we split it into two clusters $\widehat{C}^{r,1},\widehat{C}^{r,2}$ and that this results in the tree $T$. For the split the algorithm chooses, we have $T=T^{k'+1}$. 

We separate the surrogate cost into two terms: points that are in $\widehat{C}^r$ and those that are not
\begin{align*}
\surcost^{\vectmu^1,\ldots,\vectmu^k}(T) &=  \sum_{\substack{\widehat{C}^j\in T^{k'}\\ j\neq r}} \sum_{\vectx \in \widehat{C}^j} \|x - c^{\vectmu_1\ldots, \vectmu_k}(\widehat{C}^j) \|_2^2\\
&+\sum_{\vectx \in \widehat{C}^{r,1}} \|\vectx - c^{\vectmu_1\ldots, \vectmu_k}(\widehat{C}^{r,1}) \|_2^2
+\sum_{\vectx \in \widehat{C}^{r,2}} \|\vectx - c^{\vectmu_1\ldots, \vectmu_k}(\widehat{C}^{r,2}) \|_2^2,
\end{align*}
Importantly, the first term appears also in $\surcost^{\vectmu^1,\ldots,\vectmu^k}(T^{k'})$. Thus, 
$$\surcost^{\vectmu^1,\ldots,\vectmu^k}(T)- \surcost^{\vectmu^1,\ldots,\vectmu^k}(T^{k'})$$
is equal to 
$$\sum_{\vectx \in \widehat{C}^{r,1}} \|\vectx - c^{\vectmu_1\ldots, \vectmu_k}(\widehat{C}^{r,1}) \|_2^2
+\sum_{\vectx \in \widehat{C}^{r,2}} \|\vectx - c^{\vectmu_1\ldots, \vectmu_k}(\widehat{C}^{r,2}) \|_2^2- \sum_{\vectx \in \widehat{C}^{r}} \|\vectx - c^{\vectmu_1\ldots, \vectmu_k}(\widehat{C}^{r}) \|_2^2.$$
This is at most $0$ because we can give $\widehat{C}^{r,1}$ and $\widehat{C}^{r,2}$ the same center in $\vectmu^1,\ldots, \vectmu^k$ that was given to~$\widehat{C}^{r}$.
\end{proof}

\begin{proof}[Proof of Theorem~\ref{thm:algorithm_refinement}.]
  Once $T$ defines a clustering $(\widehat{C}^1,\ldots, \widehat{C}^{k'})$ that is a refinement, we know that for any cluster $i\in[k']$ each points will be assigned the same center, i.e., for all $\vectx\in \widehat{C}^i$, the value $c^{\vectmu_1\ldots, \vectmu_k}(\vectx)$ is the same. 
  Thus, 
  $$\surcost^{\vectmu^1,\ldots,\vectmu^k}(T) =
  \sum_{j=1}^{k'} \sum_{\vectx\in\widehat{C}^j} \|\vectx - c^{\vectmu_1\ldots, \vectmu_k}(\vectx)\|_2^2=
   \sum_{\vectx} \|\vectx - c^{\vectmu_1\ldots, \vectmu_k}(\vectx)\|_2^2,$$ which is exactly equal to the $k$-means reference clustering.
\end{proof}

\begin{corollary}\label{cor:surrogate_cost_after_n_stpss}
If {\em \exkmc{}} builds a tree with $n$ leaves, then it exactly matches the reference clustering, and the surrogate cost is equal to the reference cost.
\end{corollary}
\begin{proof}
%[Proof of Corollary~\ref{cor:surrogate_cost_after_n_stpss}]
If there is a step such that $T$ is a refinement of $C^{\vectmu^1,\ldots,\vectmu^k}$, then, using Theorem~\ref{thm:algorithm_refinement}, the corollary holds. 
Using Theorem~\ref{thm:surrogate-non-increasing}, the tree will continue to expand until either it is a refinement of $C^{\vectmu^1,\ldots,\vectmu^k}$ or there are no more points and the tree contains $n$ leaves. A clustering defined by $n$ leaves is one  
where each point is in a cluster of its own, which is a refinement of $C^{\vectmu^1,\ldots,\vectmu^k}.$ More specifically, for the tree $T$ that contains $n$ leaves, we see that the surrogate cost is equal to  
$$\surcost^{\vectmu^1,\ldots,\vectmu^k}(T) =  \sum_{\vectx} \|\vectx - c^{\vectmu_1\ldots, \vectmu_k}(\vectx)\|_2^2.$$ This is exactly equal to the $k$-means reference clustering.
\end{proof}

\begin{proof}[Proof of Theorem~\ref{thm:worst_case}.]
Denote by $T$ the output threshold tree of the \imm{} algorithm and by $T^{k'}$ the output threshold tree of Algorithm~\ref{algo:exkmc}. We want to show that $$\cost(T^{k'})\leq O(k^2)\cdot\cost(opt_k).$$

The proof of Theorem~3 in \cite{dasgupta2020explainable} shows that 
$\surcost(T)\leq O(k^2)\cdot\cost(opt_k).$ Together with Theorem~\ref{thm:surrogate-non-increasing}, we know that the surrogate cost is non-increasing, thus 
$$\surcost(T^{k'})\leq O(k^2)\cdot\cost(opt_k).$$
As the surrogate cost upper bounds the $k$-means cost, we know that $\cost(T^{k'})\leq\surcost(T^{k'}).$
\end{proof}

\subsection{Hard-to-explain dataset}\label{apx:synthetic_II}
In \cite{dasgupta2020explainable} a dataset was shown that cannot have an $O(1)$-approximation with any threshold tree that has $k$ leaves. In this section we will show that there is a tree with only $O(k\log k)$ leaves the achieves the optimal clustering, even on this difficult dataset. Moreover, the $\exkmc{}$ algorithm outputs such a threshold tree.  

%\michal{How about D? I'm using S and Y looks like labels in supervised learning} \cyrus{$\mathcal{D}$ sounds good to me!}

It will be useful later on to analyze the pairwise distances between the codewords. Using Hoeffding’s inequality we prove that the distances are very close to $d/2.$ 
\begin{proposition}[Hoeffding’s inequality] 
Let $X_1, ..., X_n$ be independent random variables, where for each $i$, $X_i\in[0,1]$. Define the random variable $X=\sum_{i=1}^n X_i.$ Then, for any $t\geq 0,$ $$\Pr(|X-\E[X]|\geq t)\leq 2e^{-\frac{2t^2}{n}}.$$
\end{proposition}
\begin{claim}\label{clm:synthetic_II_pairwise_distances}
Let $\delta\in(0,1)$. With probability at least $1-\delta$ over the choice of $k$ random vectors in $\{-1,1\}^d$, all the pairwise distances between the vectors are $\frac{d}2\pm2\sqrt{d}\ln\frac{k}{2\delta}.$
\end{claim}
\begin{proof}
The average distance between any two random vectors is $d/2.$ From Hoeffding’s inequality and union bound, the probability that there is a pair with distance that deviates by more than $t=2\sqrt{d}\ln\frac{k}{2\delta}$ is bounded by 
$$k^2\cdot 2e^{-\frac{2t^2}{d}}\leq \delta.$$
\end{proof}

%In \cite{dasgupta2020explainable}, Section B.5 it was shown that the height of the IMM tree for this dataset is $O(\log k)$.

For the analysis of \exkmc{} on dataset $\cD$ we will use k-means++ to find the reference centers. For completeness, we write its pseudo-code and introduce the notion of distance of point $\vectp$ from a set of points $S$ as $$
\dist(\vectp,S) = \min_{\vectx \in S}\|\vectp - \vectx\|_2^2.
$$

\begin{figure}[!htb]
  \centering
  \begin{minipage}{.6\linewidth}
    \begin{algorithm}[H]
        \SetKwFunction{EX-Means}{EX-Means}
        \SetKwFunction{ExpandTree}{ExpandTree}
        \SetKwInOut{Input}{Input}\SetKwInOut{Output}{Output}\SetKwInOut{Preprocess}{Preprocess}
        \Input{
        	$\mathcal{X}$ -- Set of vectors in $\mathbb{R}^d$\\
        	$\;\;k$ -- Number of centers\\
          }
        \Output{
            $\centersSet$ -- Set of $k$ centers
        }
        
        \LinesNumbered
        \setcounter{AlgoLine}{0}
        \BlankLine
        
        $S \leftarrow \xfunc{set}()$\;
        \While {$\lvert S\rvert < k$}
        {
            sample $\vectp\in \cX$ with probability $\frac{\mathsf{dist}(\vectp, S)}{\sum_{\vectx \in \cX}  \mathsf{dist}(\vectx, S)}$\;
            $S.\mathtt{add}(\vectp)$
        }
        run Lloyd's algorithm with centers $S$\;
        \Return $S$\;
        \caption{$k$-means++}
        \label{algo:kmeans++}
        \end{algorithm}
    \end{minipage}
\end{figure}

In Claim~\ref{clm:lower_bound_dataset_centers_k_means} we will show that one iteration of Lloyd's algorithm is enough for convergence. 
Summarizing, we consider the following form of our algorithm, that uses $k$-means++ to find the reference centers and \imm{} algorithm for the base tree:
\begin{enumerate}
    \item Use the $k$-means++ sampling procedure to choose $k$ centers from the dataset. Letting $S^i$ denote the centers at step $i$, we add a new center by sampling a point $\vectp$ with probability 
    $$ 
    \frac{\mathsf{dist}(\vectp, S^i)}{\sum_{\vectx \in \cX}  \mathsf{dist}(\vectx, S^i)}.
    $$
    \item Use one iteration of Lloyd's algorithm, where we first recluster based on the centers in $S^k$, then take the means of the new clusters $\mathcal{M}$ as the $k$ reference centers.
    \item Run the \imm{} algorithm with reference centers $\mathcal{M}$ to construct a base tree with $k$ leaves.
    \item Expand the tree to $k' > k$ leaves using \exkmc{} with reference centers $\mathcal{M}$ and the \imm{} tree. 
\end{enumerate}

%\michal{maybe I should just use norm squared}\cyrus{I'm not sure if $\|\vectp - S\|_2^2$ standard for a set $S$, so I usually use}

% \begin{theorem}\label{thm:synthetic_II_upper_bound}
% With probability at least $1-O\left(\frac{k\log k}{d}\right)$ over the randomness in $k$-means++, 
% the output of {\em \exkmc} with centers from $k$-means++, base tree of {\em \imm{}} and $k'$ desired centers, the approximation ratio is at most $$O\left(\max\left\{\log k -\frac{k'}{k},1\right\}\right)$$
% \end{theorem}
% Before we expand the tree, the last theorem states that the approximation ration is $O(\log k),$ as expected. The theorem also states that after $O(k\log k)$ steps we reach the optimal clustering. We state the last observation as a corollary. 
% \begin{corollary}
% With probability at least $1-O\left(\frac{k\log k}{d}\right)$ over the randomness in $k$-means++, 
% the output of {\em \exkmc}, with centers from $k$-means++ and base tree of {\em \imm{}}, after $O(k\log k)$ steps of the algorithm, it reaches the optimal clustering of $\cD$. 
% \end{corollary}

We first prove that after using the $k$-means++ initialization and one iteration of Lloyd’s algorithm, the resulting centers are the optimal centers, i.e., the codewords. 
\begin{claim}\label{clm:lower_bound_dataset_centers_k_means}
When running $k$-means++ on $\cD$, with probability at least $1-O(\frac{k\log k}{d})$ the resulting centers are the optimal ones.
\end{claim}
\begin{proof}
We will show that with probability at least $1-O(\frac{k\log k}{d})$, after the initialization step, the $k$-means++ algorithm took one point from each cluster. This will imply that the the optimal clusters have been found, and the optimal centers will be returned after the next step in the $k$-means++ algorithm, one iteration of Lloyd’s algorithm.   

We prove by induction on $|S|=i$ that with probability at most $\sum_{j=1}^i\frac{2j}{\alpha d(k-j)}$, the variable $S$ does not contain $i$ points from $i$ different codewords, where $\alpha d\geq d/4$ is a lower bound to the codewords pairwise distances. The base of the induction is trivial. Suppose that $|S|=i$, where $1<i<k,$ and $S$ contains points from $C^{\vectv^1},\ldots,C^{\vectv^i}.$
The distance from $S$ of each point $\vectp$ that its cluster was taken already, $\vectp\in \cup_{j=1}^iC^{\vectv^j}$, is at most $\dist(\vectp,S)\leq 2$. 
There are at most $di$ such points. 
The distance of points not taken is at least $\alpha d$. There are at least $d(k-i)$ such points.
Thus, using union bound, the probability to take a point from $\cup_{j=1}^iC^{\vectv^j}$ is at most $$\frac{2di}{\alpha d^2(k-i)}=\frac{2}{\alpha d}\cdot\frac{i}{k-i}.$$
We use the induction hypothesis and show that in total the probability to return $i+1$ points not all from different clusters is at most $$\frac{2}{\alpha d}\sum_{j=1}^i\frac{j}{k-j}.$$

Specifically, at the end, the error probability is bounded by $$
\frac{2}{\alpha d}\sum_{i = 1}^{k-1} \frac{i}{k-i} 
\leq
\frac{2}{\alpha d}\sum_{i = 1}^{k-1} \frac{k}{k-i}
=
\frac{2k}{\alpha d} \sum_{i = 1}^{k-1} \frac{1}{i}
=
O\left(\frac{k \log k}{d}\right)
$$ 
\end{proof}

In the next claim we show that if we run \imm{} with the codewords as the reference centers, then the cost is bounded by $O(dk\log k).$ The \imm{} algorithm can get the codewords as centers by running the $k$-means++ algorithm and use the previous claim.
\begin{claim}
The output threshold tree of the {\em \imm{}} algorithm on $\cD$ with the $k$ codewords as the set of reference centers, has $k$-means cost $O(dk\log k).$
\end{claim}
\begin{proof}
The work \cite{dasgupta2020explainable}  defines the mistakes of an inner node $u_{\theta,i}$ as follows: If $u_{\theta,i}$ defines a split using feature~$i$ and threshold $\theta$, a mistake is a point $\vectp$ such that $\vectp$ and its center $\vectv^j$ both reach $u_{\theta,i}$ but then they become separated and go different directions, i.e.,
$$(p_i \leq \theta \mbox{  and  } v^j_i > \theta) \quad \mbox{  or  } \quad (p_i > \theta \mbox{ and } v^j_i \leq \theta).$$   
The cost of the \imm{} tree is composed of points that are mistakes and those that are not. All the points that are not mistakes contribute in total $O(dk)$ to the cost. Each point that is a mistake contributes at most $d$ to the cost. We will show that there are only $O(k\log k)$ mistakes and this completes the proof.

We consider each level of a tree, where a {\em level} is a set of nodes with the same number of edges on the path from the root. Due to the definition of the \imm{} algorithm, each center survives in at most one node in any level. Each center $\vectv^j$ at a node $u_{\theta,i}$ can cause at most one mistake, the point in $C^{\vectv^j}$ with a zero in the $i$-th coordinate.
Thus at each level there are only $O(k)$ mistakes. From \cite{dasgupta2020explainable}, Section B.5, we know that the depth is $O(\log k).$ Thus, the total number of mistakes in the \imm{} tree is $O(k\log k)$ and the total cost is $O(dk\log k).$
\end{proof}

Putting these results together, we are now ready to prove that our algorithm produces the optimal clustering on the Synthetic II dataset $\cD$ using a tree with $O(k \log k)$ leaves.

\begin{proof}[Proof of Theorem~\ref{thm:synthetic_II_upper_bound}.]
We generalize the notion of a mistake introduced in \cite{dasgupta2020explainable}. A mistake in a leaf $u$ is a point in the leaf that its center is not $\ell(u)$, the labeling of $u$.
Note that the number of mistakes in \imm{} upper bounds the number of mistakes in the leaves at the beginning of $\exkmc$ run, which implies that at the beginning there are only $O(k\log k)$ mistakes in the leaves.

We will show that at each iteration of $\exkmc$, either there are no mistakes in the leaves or there is a split that will not incur more mistakes and actually reduce the number of mistakes. We will show that this will imply that the surrogate cost decreases by $\Omega(d)=c_1d$, for some constant $c_1>0$. 
A split that does not change the number of mistakes, or even increase it will have a smaller gain. 
Thus, number of mistakes will guarantee to decrease. This will immediately prove that after $O(k\log k)$ iterations of $\exkmc{}$ the clustering is the one defined by $\centersSet.$ We want to prove something stronger, we want to analyze the trade-off between complexity (number of leaves added) and accuracy (the surrogate cost).  

Let us analyze the approximation ratio as a function of the number of leaves added to the tree. 
If there are no mistakes, then we have found the optimal clustering. Otherwise, the total surrogate cost after \imm{} tree is $O(dk\log k)=c_2\cdot dk\log k$, for some constant $c_2$, and each new leaf will decrease the cost by $c_1\cdot d$.
Together with the optimal cost being $\Theta(kd)=c_3kd$, for some constant $c_3>0$, we deduce that if there are $k'$ leaves (equivalently, $k'-k$ iterations of $\exkmc$) and there are still mistakes, the approximation ratio is bounded by 
$$\frac{c_2dk\log k-c_1d(k'-k)}{c_3 dk}=\frac{c_2}{c_3}\cdot\log k -\frac{c_1}{c_3}\cdot\frac{k'}{k}+\frac{c_1}{c_3} $$
In different words, there is a constant $c>0$ such that the approximation ratio is bounded by $$O\left(\max\left\{\log k -c\cdot \frac{k'}{k},\ 1\right\}\right).$$

%$$O\left(\max\left\{\log k  -\frac{k'-k}{k},\ 1\right\}\right)=O\left(\max\left\{\log k -\frac{k'}{k},\ 1\right\}\right).$$
% Thus, after $O(k\log k),$ there are no more mistakes. 

Let us prove that if there is a mistake in a leaf, then some split decreases the surrogate cost by $\Omega(d).$
If there is  leaf $u$ with a point $\vectp$ that is a mistake, then $\vectp$ does not belong to the codeword $v^{\ell(u)}.$ 
There is a feature $i$ where $p_i\neq 0$ and $p_i\neq v^{\ell(u)}_i$ (actually there are $\alpha d-1$ such coordinates, and any one of them will be good). Without loss of generality, assume that $p_i=-1$. 
Focus on the split $(i,\theta)$ with $\theta=-0.5.$ 
Denote by $C^{\vectp},C^{\neg\vectp}$ the partition of points in $u$ defined by this split. Suppose that $\vectp \in C^{\vectp}.$
Denote by $C^{\ell(u)}$ all the points in $u$ that belong to the cluster $\ell(u).$
This choice of threshold ensures that all points in $C^{\ell(u)}$ are in $C^{\neg\vectp}$ and not in $C^{\vectp}$. 
This means that all points in $C^{\ell(u)}$ are now in a different cluster than $\vectp$ after this split. In different words, all points that were not mistakes as they were in $C^{\ell(u)}$ will remain as non-mistakes. 
Furthermore, this split will make $\vectp$ a non-mistake.
In $C^{\vectp}$ there are at most $O(k\log k)$ points (because only mistakes can be in $C^{\vectp}$). 
Using Claim~\ref{clm:synthetic_II_pairwise_distances}, their cost from changing the center can increase by at most $O(\sqrt{d}\log k)$,   
and they decrease the surrogate cost by $\Omega(d).$ 
Thus the decrease in this split is at least $\Omega(d-\sqrt{d}k\log^2k)=\Omega(d).$
As the $\exkmc$ algorithm chooses the split minimizing the surrogate cost, the cost decreases by $\Omega(d)$ in this step.  

The last thing we need to show is that for every split with $\Omega(d)$ gain, the number of mistakes must go down. 
Suppose the algorithm made a split at node $u$ where previously there were $n_u$ points and $m_u$ mistakes. After the split there are $n_r$ nodes on the right successor and $n_\ell$ on the left, $n_r+n_\ell=n_u$ and $m_r$ mistakes on the right and $m_\ell$ mistakes on the left. 

Letting $\Delta=2\sqrt{d}\ln\frac{k}{10}=\Theta(\sqrt{d})$, the cost of each node is, by Claim~\ref{clm:synthetic_II_pairwise_distances}, $$\# \text{mistakes}\cdot \left(\frac{d}{2}\pm\Delta\right) + \#\text{non-mistakes}.$$
Thus, the maximal gain of a split is achieved in case $u$ has the largest cost and its successors the smallest. Specifically, the maximal gain is 
$$\left[m_u\left(\frac{d}{2}+\Delta\right)+(n_u-m_u)\right]-\left[m_r\left(\frac{d}{2}-\Delta\right)+(n_r-m_r)\right]-\left[m_\ell\left(\frac{d}{2}-\Delta\right)+(n_\ell-m_\ell)\right].$$
The last term is equal to $$(\Delta-1)(m_u-m_r-m_\ell).$$
For the gain to be $\Omega(d)$, or even positive, the number of mistakes must decrease.
\end{proof}

%%%%%%%%%%%%%%%%%%%

\section{More Experimental Details}
\label{sec:appendix_setup}

\subsection{Algorithms and Baselines}

Our empirical evaluation compared the following clustering methods.

\begin{itemize}
    \item CART~\cite{breiman1984classification}: Each data point was assigned with a class label, based on the clustering result of the near-optimal baseline. We have used \sklearn{} implementation of decision tree minimizing the \texttt{gini} impurity. Number of leaves was controlled by \texttt{max\_leaf\_nodes} parameter. 
    
    \item KDTree~\cite{bentley1975multidimensional}: For each tree node the best cut was chosen by taking the coordinate with highest variance, and split according to the median threshold. The size of the constructed tree is controlled through \texttt{leaf\_size} parameter, since the splits are always balanced we obtained a tree with up to $k'$ leaves by constructing a KDTree with $\texttt{leaf\_size}=\lfloor \frac{n}{k'} \rfloor$. Each of the $k'$ leaves was labeled with a cluster id from $1$ to $k$, such the $\surcost$ will be minimized with the centers of the near-optimal baseline.
    
    \item CUBT \cite{fraiman2013interpretable}: Constructed clustering tree using \cubt{} R package \cite{cubtR}. \cubt{} algorithm is composed out of three steps: (1) build a large tree, (2) prune branches, and (3) join tree leaves labels. Each step is controlled by multiple parameters. For each step we applied grid-search over multiple parameters, and selected the parameters that minimize the final tree $\cost$. The hyper-parameters were chosen  based on \cite{fraiman2013interpretable} recommendation, and available in \ref{sec:appendix_hyper_param}. To construct a tree with $k'$ leaves the we have set $\texttt{nleaves}=k'$ in the prune step, and to verify that only $k$ clusters will be constructed we have set $\texttt{nclass}=k$ in the join step.
    
    \item CLTree \cite{liu2005clustering}: Constructed clustering tree using \cltree{} Python package \cite{cltreeP}.
    \cltree{} algorithm first construct a large tree with up to \texttt{min\_split} samples in each leaf, and afterwards prune its branches. Pruning step is controlled by \texttt{min\_y} and \texttt{min\_rd} parameters. We applied grid-search over those parameters (values specified in \ref{sec:appendix_hyper_param}), for each combination we counted the number of leaves, and for each number of leaves we have taken the tree with minimal $\cost$. 
    
    \item ExKMC: Applied our proposed method for tree construction (Algorithm \ref{algo:exkmc}), that minimize $\surcost$ at each split with the centers of the near-optimal baseline.
    
    \item ExKMC (base: IMM): Constructed a base tree with $k$ leaves according to the \imm{} algorithm \cite{dasgupta2020explainable}. The tree was expanded with our proposed method (Algorithm \ref{algo:exkmc}) that minimize $\surcost$ at each split. Both \imm{} and the \exkmc{} method used the centers of the near-optimal baseline.
\end{itemize}

For each combination of parameters, if execution time was more than 1 hour, the execution was terminated and ignored. Execution termination occurred only with \cubt{} and \cltree{} over the larger datasets.

\subsection{Datasets}
\label{sec:appendix_datasets}
Datasets in the empirical evaluation are depicted in Table \ref{tab:datasets}.

\begin{table}[!htb]
    \centering
    \caption{Datasets properties}
    \label{tab:datasets}
    \begin{tabular}{l l l l}
    \toprule
    Dataset & $k$ & $n$ & $d$ \\
    \midrule
    \multicolumn{4}{c}{\textbf{Small Datasets}}\\
    \midrule
    Iris \cite{fisher1936use}& 3 & 150 & 4 \\
    Wine \cite{Dua:2019}& 3 & 178 & 13 \\
    Breast Cancer \cite{Dua:2019}& 2 & 569 & 30 \\
    Digits \cite{lecun1998gradient}& 10 & 1,797 & 64 \\    
    Mice Protein \cite{higuera2015self} & 8 & 1,080 & 69 \\
    Anuran Calls \cite{mendoza2019morphological} & 10 & 7,195 & 22 \\
    \midrule
    \multicolumn{4}{c}{\textbf{Larger Datasets}}\\
    \midrule
    Avila \cite{de2018reliable} & 12 & 20,867 & 10 \\
    Covtype \cite{blackard1999comparative}& 7 & 581,012 & 54 \\
    20 Newsgroups \cite{joachims1996probabilistic}& 20 & 18,846 & 1,893 \\
    CIFAR-10 \cite{krizhevsky2009learning}& 10 & 50,000 & 3,072 \\
    \midrule
    \multicolumn{4}{c}{\textbf{Synthetic Datasets}}\\
    \midrule
    Synthetic I & 3 & 5,000 & 1,000 \\
    Synthetic II & 30 & 30,000 & 1,000 \\
    \bottomrule
    \end{tabular}
\end{table}

Categorical values and class labels were removed from the datasets. The number of clusters, $k$, was set to be equal to the number of class labels.
For the 20newsgroups dataset, documents were converted to vectors by removal English stop-words and construction of word count vectors, ignoring rare terms (with document frequency smaller than $1\%$).

\smallpar{Synthetic I dataset} The dataset of synthetic I is a slight adaptation of the one described in \cite{dasgupta2020explainable}, which was designed the highlight the weakness of \cart{} algorithm. It contains $5,000$ points:
\begin{itemize}
    \item Two of them are $(\nu,1,1,\ldots, 1)$ and $(\nu, 0,0,\ldots,0),$ where $\nu=1000$ is a large number.
    \item  Half of the remaining points are $0$ in the first feature and another random $100$ features are also $0$, all the remaining features are $1.$
    \item  The remaining points also have zero in the first feature, and the other random $100$ features are set to~$1$ and the rest of the features are $0$. 
\end{itemize}

Properties of synthetic II are described in  Section~\ref{apx:synthetic_II}.

\subsection{Hyper-parameters}
\label{sec:appendix_hyper_param}

Grid search was executed on the following hyper-parameters values:
\begin{itemize}
    \item CUBT:\\
    $\texttt{minsize} \in [5, 10, 15]$\\
    $\texttt{mindev} \in [0.001, 0.7, 0.9]$\\
    $\texttt{mindist} \in [0.3, 0.5]$\\
    $\texttt{alpha} \in [0.2, 0.4, 0.6]$\\
    To construct a tree with $k'$ leaves the we have set $\texttt{nleaves}=k'$ in the prune step, and to verify that only $k$ clusters will be constructed we have set $\texttt{nclass}=k$ in the join step.
    
    \item CLTree:\\
    $\texttt{min\_split} \in [10, 20, 30, 50, 100, 200, 500, 1000, 1500, 2000]$ \\
    $\texttt{min\_y} \in [0.1, 0.2, 0.3, 0.4, 0.5, 1, 1.5, 2, 2.5, 3, 3.5, 4, 4.5, 5]$\\
    $\texttt{min\_rd} \in [1, 2, 3, 4, 5, 6, 7, 8, 9, 10, 15, 20, 25, 30, 50]$
\end{itemize}

% \section{Runtime}

\subsection{Extra experiments}
\label{sec:appendix_extra_exp}
\label{sec:appendix_runtime}

Figure \ref{fig:runtime_single_process} depicted the run time of \exkmc{} constricting a tree with $2k$ leaves using a single process. 
The operations of \imm{} and \exkmc{} can be feature-wise paralleled, where \texttt{KMeans} iterations can also be executed in parallel.  
Figure \ref{fig:appendix_run_time} compare the running times of single processor and four processor. Using four process \exkmc{} constructs tree with $20$ leaves over CIFAR-10 dataset in less than $7$ minutes. 

\begin{figure}[!htb]
    \centering
    \begin{subfigure}{.5\linewidth}
    \centering
    \includegraphics[width=.8\linewidth]{figures/run_time_10_raw.png}
    \caption{Single process}
    \label{fig:run_time_single}
    \end{subfigure}%
    \begin{subfigure}{.5\linewidth}
    \centering
    \includegraphics[width=.8\linewidth]{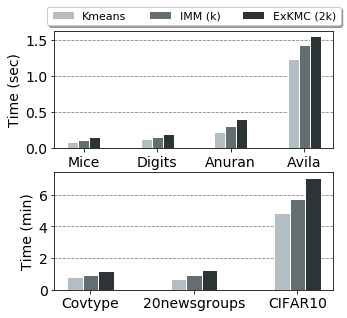}
    \caption{Four processes}
    \label{fig:run_time_njobs4}
    \end{subfigure}
    \caption{Running times of constructing a tree-based clustering with $2k$ leaves. We consider both a single processor (left) and a parallel version using four processes (right). The y-axis labels differ between left and right graphs. Overall, parallelism improves the running time of \exkmc{} by 2--3$\times$.}
    \label{fig:appendix_run_time}
\end{figure}

% \subsection{Small datasets scaled}
\label{sec:appendix_experiments_small_scaled}

% \subsection{CIFAR-10 with more leaves}

\begin{figure}[!htb]
        \centering
        \includegraphics[width=.7\linewidth]{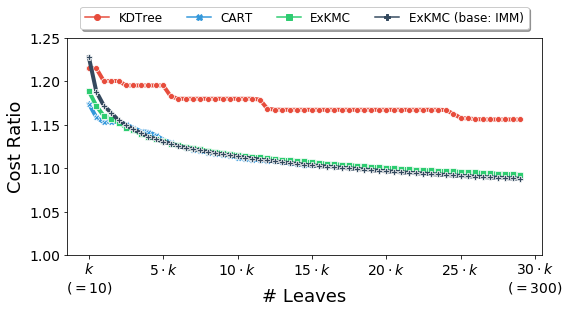}
        \caption{{\bf CIFAR-10 Convergence Rate.} We separately showcase CIFAR-10 because it is an exception to many trends. The algorithms fail to converge to a cost ratio of 1.0 compared to the reference clustering. This is likely because pixels are not good features in this context, since the trees only use a subset of them, one at a time. We also notice that IMM starts with the worst performance, but when expanded using \exkmc{}, eventually outperforms the competitors (for $k' > 10k$). }
        \label{fig:cifar10_30k}
\end{figure}

\begin{figure}
%\begin{table}
    \begin{tabular}{c c c c}       
        \multicolumn{4}{l}{\textbf{Zooming in on the Cost for Small Datasets}}\\
        \begin{subfigure}{.24\linewidth}
        \centering
        \includegraphics[width=\linewidth]{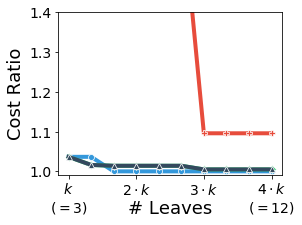}
        \caption{\scriptsize{Iris}}
        \label{fig:iris_scale}
        \end{subfigure}& 
        \begin{subfigure}{.24\linewidth}
        \centering
        \includegraphics[width=\linewidth]{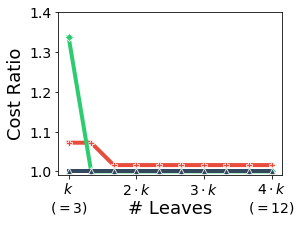}
        \caption{\scriptsize{Wine}}
        \label{fig:wine_scale}
        \end{subfigure}& 
        \multicolumn{2}{c}{
         \begin{subfigure}{.4\linewidth}
            \centering
           \includegraphics[width=\linewidth]{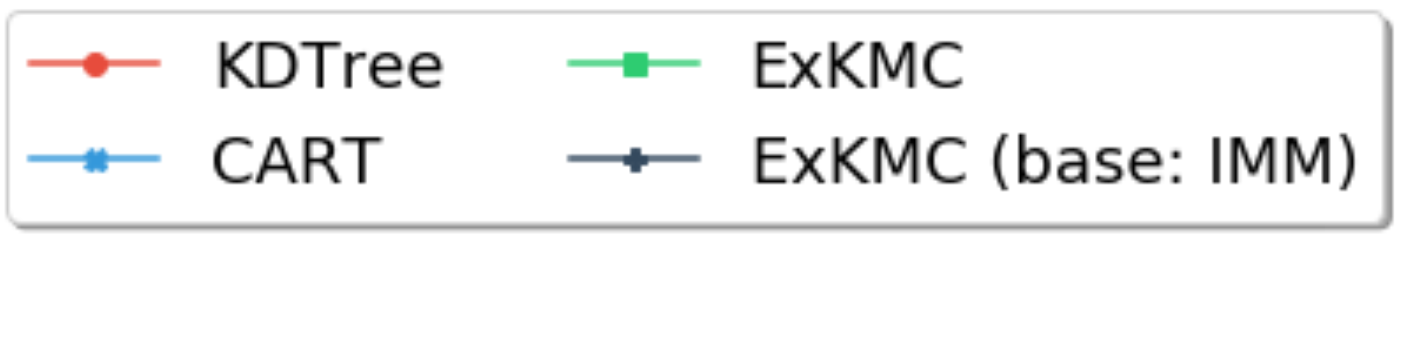}
        \end{subfigure}
        }\\
        \begin{subfigure}{.24\linewidth}
        \centering
        \includegraphics[width=\linewidth]{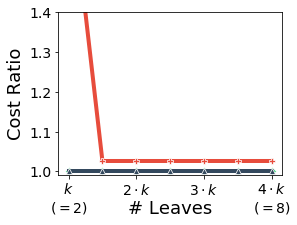}
        \caption{\scriptsize{Breast Cancer}}
        \label{fig:brease_cancer_scale}
        \end{subfigure}& 
        \begin{subfigure}{.24\linewidth}
        \centering
        \includegraphics[width=\linewidth]{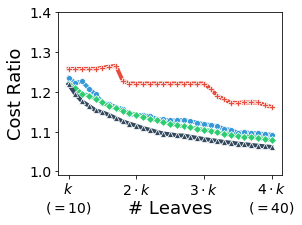}
        \caption{\scriptsize{Digits}}
        \label{fig:digits_scale}
        \end{subfigure}&
        \begin{subfigure}{.24\linewidth}
        \centering
        \includegraphics[width=\linewidth]{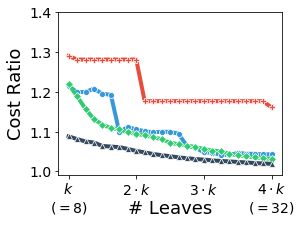}
        \caption{\scriptsize{Mice Protein}}
        \label{fig:mice_protein_scale}
        \end{subfigure}&
        \begin{subfigure}{.24\linewidth}
        \centering
        \includegraphics[width=\linewidth]{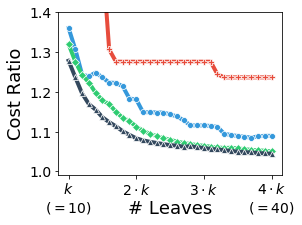}
        \caption{\scriptsize{Anuran}}
        \label{fig:anurn_scale}
        \end{subfigure}
    \end{tabular}
    \caption{
        Results of small datasets presented in Figure \ref{fig:experiments} where $\cost{}$ is in range $1.0 - 1.4$.
    }
    \label{fig:experiments_small_scaled}
\end{figure}

\begin{figure}
    \renewcommand{\arraystretch}{1.3}
    \begin{tabular}{c c c c}
        \rowcolor{gray!15}\multicolumn{4}{l}{\textbf{Small Datasets \scriptsize ($k$-means cost ratio, comparing five initializations)}}\\
        \begin{subfigure}{.24\linewidth}
        \centering
        \includegraphics[width=\linewidth]{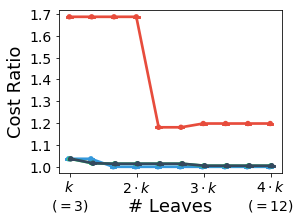}
        \caption{\scriptsize{Iris}}
        \label{fig:iris_agg}
        \end{subfigure}& 
        \begin{subfigure}{.24\linewidth}
        \centering
        \includegraphics[width=\linewidth]{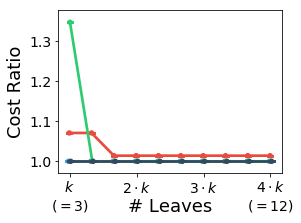}
        \caption{\scriptsize{Wine}}
        \label{fig:wine_agg}
        \end{subfigure}& 
        \multicolumn{2}{c}{
            \begin{subfigure}{.4\linewidth}
            \centering
            \includegraphics[width=\linewidth]{figures/legend_acc.PNG}
        \end{subfigure}
        }\\
        \begin{subfigure}{.24\linewidth}
        \centering
        \includegraphics[width=\linewidth]{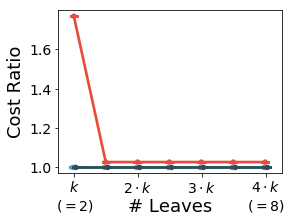}
        \caption{\scriptsize{Breast Cancer}}
        \label{fig:brease_cancer_agg}
        \end{subfigure}& 
        \begin{subfigure}{.24\linewidth}
        \centering
        \includegraphics[width=\linewidth]{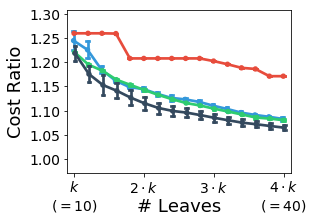}
        \caption{\scriptsize{Digits}}
        \label{fig:digits_agg}
        \end{subfigure}&
        \begin{subfigure}{.24\linewidth}
        \centering
        \includegraphics[width=\linewidth]{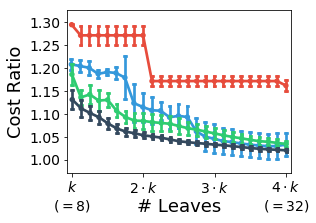}
        \caption{\scriptsize{Mice Protein}}
        \label{fig:mice_protein_agg}
        \end{subfigure}&
        \begin{subfigure}{.24\linewidth}
        \centering
        \includegraphics[width=\linewidth]{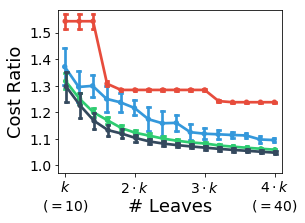}
        \caption{\scriptsize{Anuran}}
        \vspace{1ex}
        \label{fig:anurn_agg}
        \end{subfigure}\\
        %\midrule
        \rowcolor{gray!15}\multicolumn{4}{l}{\textbf{Larger Datasets \scriptsize ($k$-means cost ratio, comparing five initializations)}}\\
        \begin{subfigure}{.24\linewidth}
        \centering
        \includegraphics[width=\linewidth]{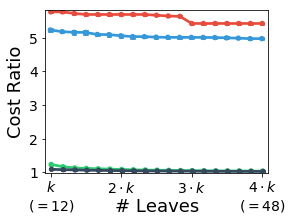}
        \caption{\scriptsize{Avila}}
        \label{fig:avila_agg}
        \end{subfigure}&
        \begin{subfigure}{.24\linewidth}
        \centering
        \includegraphics[width=\linewidth]{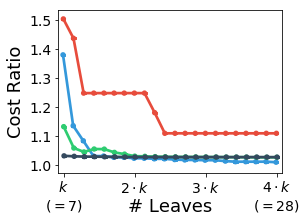}
        \caption{\scriptsize{Covtype}}
        \label{fig:covtype_agg}
        \end{subfigure}&
        \begin{subfigure}{.24\linewidth}
        \centering
        \includegraphics[width=\linewidth]{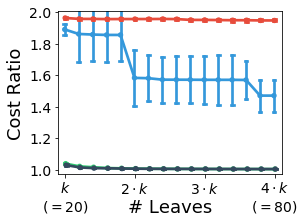}
        \caption{\scriptsize{20newsgroups}}
        \label{fig:20newsgroups_agg}
        \end{subfigure}&
        \begin{subfigure}{.24\linewidth}
        \centering
        \includegraphics[width=\linewidth]{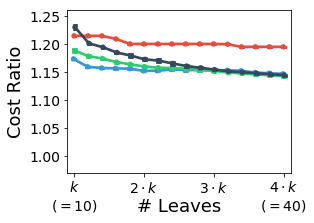}
        \caption{\scriptsize{CIFAR-10}}
        \vspace{1ex}
        \label{fig:cifar10_agg}
        \end{subfigure}\\
        %\midrule
        \rowcolor{gray!15}\multicolumn{4}{l}{\textbf{Synthetic Datasets \scriptsize \scriptsize ($k$-means cost ratio, comparing five initializations)}}\\
        &
        \begin{subfigure}{.24\linewidth}
        \centering
        \includegraphics[width=\linewidth]{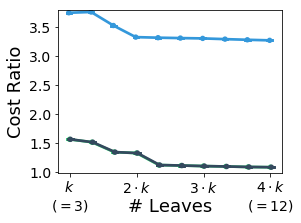}
        \caption{\scriptsize{Synthetic I}}
        \label{fig:syntactic1_agg}
        \end{subfigure}&
        \begin{subfigure}{.24\linewidth}
        \centering
        \includegraphics[width=\linewidth]{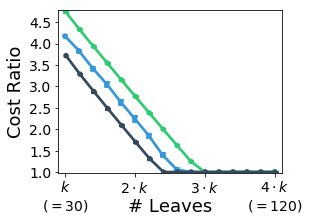}
        \caption{\scriptsize{Synthetic II}}
        \label{fig:syntactic2_agg}
        \end{subfigure}&\\
    \end{tabular}
    \caption{
        Aggregation of results presented in Figure \ref{fig:experiments} over five different executions of \texttt{KMeans}. All tree construction algorithms that were examined are deterministic, but different initialization of centers to \texttt{KMeans} may result in different reference clustering. Points are the median of the five executions, and error bars are the standard deviation. In many datasets, the deviations are negligible or even equal to $0$. When the number of leaves surpasses $2k$ the results of all algorithms remain separated in the presence of deviations, verifying the significance and reproducibility of the results.
    }
    \label{fig:experiments_error_bars}
\end{figure}

\begin{figure}[!ht]
    \centering
    \begin{subfigure}{.24\linewidth}
    \centering
    \includegraphics[width=\linewidth]{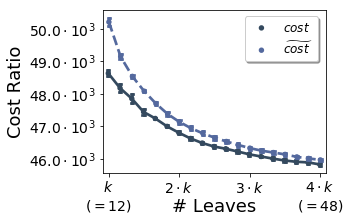}
    \caption{\scriptsize{Avila}}
    \label{fig:avila_sur_agg}
    \end{subfigure}%
    \begin{subfigure}{.24\linewidth}
    \centering
    \includegraphics[width=\linewidth]{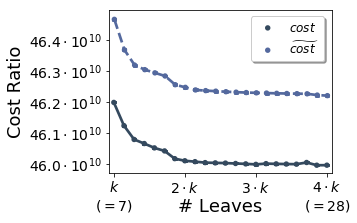}
    \caption{\scriptsize{Covtype}}
    \label{fig:covtype_sur_agg}
    \end{subfigure}%
    \begin{subfigure}{.24\linewidth}
    \centering
    \includegraphics[width=\linewidth]{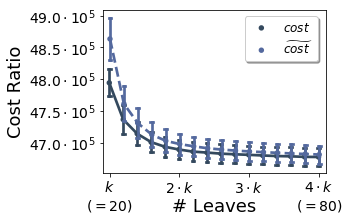}
    \caption{\scriptsize{20newsgroups}}
    \label{fig:20newsgroups_sur_agg}
    \end{subfigure}%
    \begin{subfigure}{.24\linewidth}
    \centering
    \includegraphics[width=\linewidth]{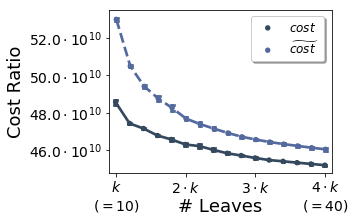}
    \caption{\scriptsize{CIFAR-10}}
    \label{fig:cifar10_sur_agg}
    \end{subfigure}%
    \caption{
    Aggregation of results presented in Figure \ref{fig:suroogate_vs_optimal} over five different executions of \texttt{KMeans}. For three of the datasets, the deviations are negligible. For 20newsgroups, we see that both the $k$-means and the surrogate cost have relatively higher variance. Overall, taking the best clustering of the five initialization leads to consistent and reproducible results.
    }
    \label{fig:suroogate_vs_optimal_agg}
\end{figure}

% \section{Accuracy}
\label{sec:appendix_accuracy}

% Figure \ref{fig:experiments_accuracy} depict the tree accuracy, with respect to the reference centers, as a factor of the number of leaves.

\begin{figure}
    \renewcommand{\arraystretch}{1.3}
    \begin{tabular}{c c c c}
        \rowcolor{gray!15}\multicolumn{4}{l}{\textbf{Small Datasets \scriptsize (Clustering accuracy)}}\\
        \begin{subfigure}{.24\linewidth}
        \centering
        \includegraphics[width=\linewidth]{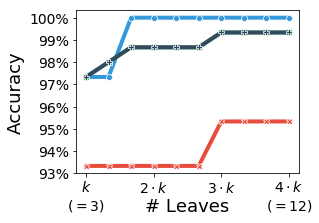}
        \caption{\scriptsize{Iris}}
        \label{fig:iris_acc}
        \end{subfigure}& 
        \begin{subfigure}{.24\linewidth}
        \centering
        \includegraphics[width=\linewidth]{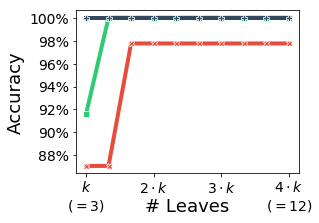}
        \caption{\scriptsize{Wine}}
        \label{fig:wine_acc}
        \end{subfigure}& 
        \multicolumn{2}{c}{
            \begin{subfigure}{.4\linewidth}
            \centering
            \includegraphics[width=\linewidth]{figures/legend_acc.PNG}
        \end{subfigure}
        }\\
        \begin{subfigure}{.24\linewidth}
        \centering
        \includegraphics[width=\linewidth]{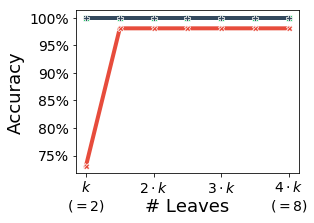}
        \caption{\scriptsize{Breast Cancer}}
        \label{fig:brease_cancer_acc}
        \end{subfigure}& 
        \begin{subfigure}{.24\linewidth}
        \centering
        \includegraphics[width=\linewidth]{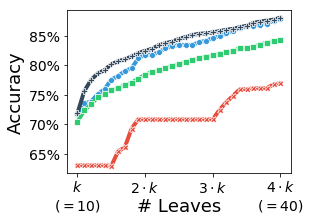}
        \caption{\scriptsize{Digits}}
        \label{fig:digits_acc}
        \end{subfigure}&
        \begin{subfigure}{.24\linewidth}
        \centering
        \includegraphics[width=\linewidth]{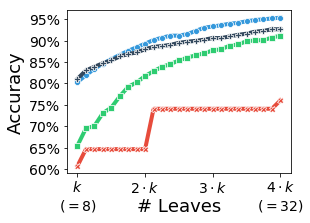}
        \caption{\scriptsize{Mice Protein}}
        \label{fig:mice_protein_acc}
        \end{subfigure}&
        \begin{subfigure}{.24\linewidth}
        \centering
        \includegraphics[width=\linewidth]{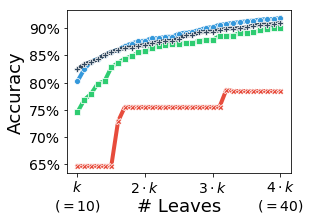}
        \caption{\scriptsize{Anuran}}
        \vspace{1ex}
        \label{fig:anurn_acc}
        \end{subfigure}\\
        %\midrule
        \rowcolor{gray!15}\multicolumn{4}{l}{\textbf{Larger Datasets \scriptsize (Clustering accuracy)}}\\
        \begin{subfigure}{.24\linewidth}
        \centering
        \includegraphics[width=\linewidth]{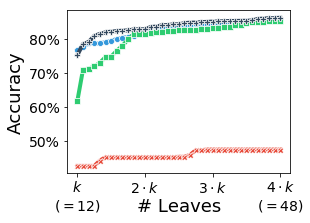}
        \caption{\scriptsize{Avila}}
        \label{fig:avila_acc}
        \end{subfigure}&
        \begin{subfigure}{.24\linewidth}
        \centering
        \includegraphics[width=\linewidth]{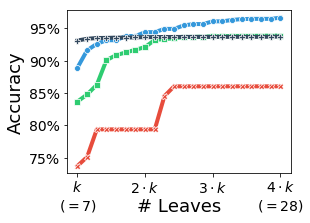}
        \caption{\scriptsize{Covtype}}
        \label{fig:covtype_acc}
        \end{subfigure}&
        \begin{subfigure}{.24\linewidth}
        \centering
        \includegraphics[width=\linewidth]{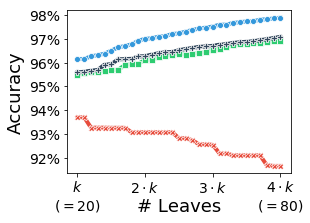}
        \caption{\scriptsize{20newsgroups}}
        \label{fig:20newsgroups_acc}
        \end{subfigure}&
        \begin{subfigure}{.24\linewidth}
        \centering
        \includegraphics[width=\linewidth]{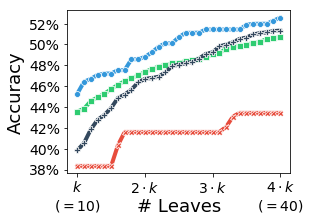}
        \caption{\scriptsize{CIFAR-10}}
        \vspace{1ex}
        \label{fig:cifar10_acc}
        \end{subfigure}\\
        %\midrule
        \rowcolor{gray!15}\multicolumn{4}{l}{\textbf{Synthetic Datasets \scriptsize (Clustering accuracy)}}\\
        &
        \begin{subfigure}{.24\linewidth}
        \centering
        \includegraphics[width=\linewidth]{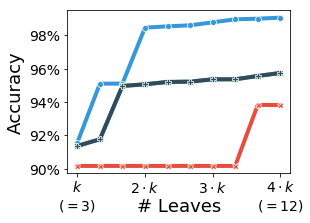}
        \caption{\scriptsize{Synthetic I}}
        \label{fig:syntactic1_acc}
        \end{subfigure}&
        \begin{subfigure}{.24\linewidth}
        \centering
        \includegraphics[width=\linewidth]{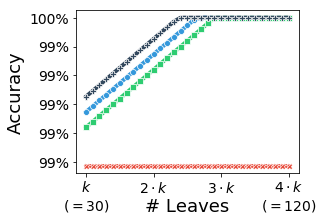}
        \caption{\scriptsize{Synthetic II}}
        \label{fig:syntactic2_acc}
        \end{subfigure}&\\
    \end{tabular}
    \caption{
        Results of clustering accuracy experiments over 12 different datasets. Graphs depict the accuracy of the tree-based clustering with respect to the reference $k$-means ($y$-axis, best result is $100\%$) as function of the number of tree leaves ($x$-axis).
    }
    \label{fig:experiments_accuracy}
\end{figure}

\end{document}